\documentclass[prodmode,acmtkdd]{acmsmall}

\usepackage[ruled]{algorithm2e}
\usepackage{amsmath, amssymb, amsfonts}
\usepackage{epstopdf}
\usepackage{multirow}
\usepackage{hyperref}
\usepackage{url}
\usepackage{graphicx}

\renewcommand{\algorithmcfname}{ALGORITHM}
\SetAlFnt{\small}
\SetAlCapFnt{\small}
\SetAlCapNameFnt{\small}
\SetAlCapHSkip{0pt}
\IncMargin{-\parindent}

\acmVolume{2}
\acmNumber{4}
\acmArticle{01}
\acmYear{2012}
\acmMonth{4}

\renewcommand{\v}[1]{\ensuremath{\mathbf{#1}}} 
\newcommand{\gv}[1]{\ensuremath{\mbox{\boldmath$ #1 $}}} 
 
\newcommand{\grad}[1]{\gv{\nabla} #1} 
\let\baraccent=\= 
\renewcommand{\=}[1]{\stackrel{#1}{=}} 

\newcommand{\vect} [1] {\boldsymbol{#1}}

\begin{document}

\markboth{A. Acharya et al.}{An Optimization Framework for Semi-Supervised and Transfer Learning using Multiple Classifiers and Clusterers}

\title{An Optimization Framework for Semi-Supervised and Transfer Learning using Multiple Classifiers and Clusterers}
\author{Ayan Acharya
\affil{University of Texas at Austin, USA}
Eduardo R. Hruschka
\affil{University of Texas at Austin, USA; University of Sao Paulo at Sao Carlos, Brazil}
Joydeep Ghosh
\affil{University of Texas at Austin, USA}
Sreangsu Acharyya
\affil{University of Texas at Austin, USA}}

\begin{abstract}
Unsupervised models can provide supplementary soft constraints to help classify new, ``target'' data since similar instances in the target set are more 
likely to share the same class label. Such models can also help detect possible differences between training and target distributions, which is useful in 
applications where concept drift may take place, as in transfer learning settings. This paper describes a general optimization framework that takes as input class 
membership estimates from existing classifiers learnt on previously encountered ``source'' data, as well as a similarity matrix from a 
cluster ensemble operating solely on the target data to be classified, and yields a consensus labeling of the target data. This framework admits a wide range of loss functions and classification/clustering methods. It exploits properties of Bregman divergences in conjunction with Legendre duality
 to yield a principled and scalable approach. A variety of experiments show that the proposed framework can yield results substantially superior  to those provided by popular transductive learning techniques or by na\"{\i}vely applying classifiers learnt on the original task to the target data.

\end{abstract}




\acmformat{Acharya, A., Hruschka, E.R., Ghosh, J., Acharyya, S. 2012.}

\begin{bottomstuff}
This work has been supported by NSF Grants (IIS-0713142 and IIS-1016614) and by the Brazilian Research Agencies FAPESP and CNPq. 
Author's addresses: A. Acharya, Department of Electrical and Computer Engineering,
University of Texas at Austin; 
\end{bottomstuff}

\maketitle

\section{Introduction}
\label{Intro}

In several data mining applications, ranging from identifying distinct control regimes in complex plants to characterizing different types 
of stocks in terms of price and volume movements, one builds an initial classification model that needs to be applied to unlabeled data acquired 
subsequently. Since the statistics of the underlying phenomena being modeled often changes with time, these classifiers may also need to be occasionally 
rebuilt if performance degrades beyond an acceptable level. In such situations, it is desirable that the classifier functions well with as little 
labeling of new data as possible, since labeling can be expensive in terms of time and money, and it is a potentially error-prone process. Moreover, the 
classifier should be able to adapt to changing statistics to some extent, given the afore-mentioned constraints. 

This paper addresses the problem of combining multiple classifiers and clusterers in a fairly general setting, that includes the scenario sketched above.  
An  ensemble of classifiers is first learnt on an initial labeled training dataset which can conveniently be denoted by ``source'' dataset. 
At this point, the training data can be discarded.  
Subsequently, when new, unlabeled target data is encountered, a cluster ensemble is applied to it to yield a similarity 
matrix. In addition, the previously learnt classifier(s) can be used to obtain an estimate of the  class probability distributions 
for this data. The heart of our technique is an optimization framework that combines both sources of information to yield a 
consensus labeling of the target data. General properties of a large class of loss functions described by Bregman divergences are 
exploited in this framework in conjunction with Legendre duality and  a
notion of variable splitting that is also used in alternating direction method of multipliers \cite{bopc11}) to yield a principled and scalable solution. 

Note that the setting described above is different from transductive learning setups where both labeled and unlabeled data are available 
at the same time for model building \cite{sibe08}, as well as online methods where decisions are made on one new example at a time, and after each such decision,  the true label of 
the example is obtained and used to update the model parameters \cite{blum98}. Additional differences from existing approaches are described in the 
section on related works. For the moment we note that 
the underlying assumption is that similar new instances in the target set are more likely to share the same class label. Thus, the supplementary 
constraints provided by the cluster ensemble can be useful for improving the generalization capability of the resulting classifier system, specially 
when labeled data for training the base classifiers is scarce. Also, these supplementary constraints provided by unsupervised models can be useful for 
designing learning methods that help determine differences between training and target distributions, making the overall system more robust against concept 
drift. To highlight these additional capabilities that are useful for transfer learning, we provide a separate set of empirical studies where the target 
data is related to but significantly different from the initial training data.
 
The remainder of this paper is organized as follows. After addressing related work in Section \ref{sec:relatedwork}, the proposed optimization framework 
and its associated algorithm --- named \textbf{OAC\textsuperscript{3}}, from \textbf{O}ptimization \textbf{A}lgorithm for \textbf{C}ombining \textbf{C}lassifiers and \textbf{C}lusterers --- are described in 
Section \ref{sec:C3E}. 
This particular algorithm has been briefly introduced in \cite{achg11}.
A convergence analysis of \textbf{OAC\textsuperscript{3}} is reported in Section \ref{sec:convergence}, while Section \ref{rateofconv} 
analyses its convergence rate. An experimental study illustrating the potential of the proposed framework for a 
variety of applications is reported in Section \ref{sec:exp}. Finally, Section \ref{sec:conc} concludes the paper.

\textbf{Notation}. Vectors and matrices are denoted by bold faced lowercase and capital letters, respectively. Scalar 
variables are written in italic font. A set is denoted by a calligraphic uppercase letter. The effective domain of a function $f(y)$, \textit{i.e.}, the set of 
all $y$ such that $f(y)<+\infty$ is denoted by $\text{dom}(f)$, while the interior and the relative interior of a set $\mathcal{Y}$ are denoted by int($\mathcal{Y}$) 
and ri($\mathcal{Y}$), respectively. For $\mathbf{y}_{i}, \mathbf{y}_{j}\in\mathbb{R}^{k}$, $\langle\mathbf{y}_{i},\mathbf{y}_{j}\rangle$ denotes their inner product.
A function $f\in C^{k^{\prime}}$ if all of its first $k^{\prime}$ derivatives exist and are continuous.

\section{Related Work}
\label{sec:relatedwork}

This contribution leverages the theory of classifier and cluster ensemble to solve transfer and semi-supervised learning problems. 
Also, the underlying optimization framework inherits properties from alternating optimization type of algorithms. 
In this section, a brief introduction to each of these different research areas is provided. 


The combination of multiple single or base classifiers to generate a more capable ensemble classifier has been an active area of research for the past two decades [\citeNP{kunc04}; \citeNP{oztu08}].  Several papers provide both 
theoretical results \cite{tugh96} and empirical evidence showing the utility of such approaches for solving difficult classification problems. 
For instance, an analytical framework to mathematically quantify the improvements in classification results due to combining multiple models has been addressed in \cite{tugh96}. 
A survey of traditional ensemble techniques --- including their applications to many difficult real-world problems such as remote sensing, person recognition, one vs. all recognition, and medicine --- is presented in \cite{oztu08}. 
In summary, the extensive literature on the subject has shown that an ensemble created from diversified classifiers is typically more accurate than its individual components. 

Analogously, 
several research efforts have shown that cluster ensembles can improve the quality of results as compared to a single clustering solution --- {\it e.g.}, see \cite{hoha11,ghac11} and references therein.
Indeed,   the potential motivations and benefits for using cluster ensembles are much broader than those for using classifier ensembles, 
for which improving the predictive accuracy is usually the primary goal. More specifically, cluster ensembles can be used to generate more robust and stable clustering results (compared to a single clustering approach), 
perform distributed computing under privacy or sharing constraints, or reuse existing knowledge \cite{stgh02b}.
We note however that:
\begin{itemize}
 \item Like single classifiers/clusterers, with very few exceptions \cite{poli07}, 
 ensemble methods assume that the test or scoring data comes from the same underlying distribution as the training (and validation) data. Thus their performance degrades if the underlying input-output map changes over time.
 \item There is relatively little work in incorporating both labeled and unlabeled data while building ensembles, in contrast to the substantial amount of recent interest in semi-supervised learning - including semi-supervised 
clustering, semi-supervised classification, clustering with constraints and transductive learning methods - using a single model [\citeNP{chsz06}; \citeNP{zhgo09}; \citeNP{cacz09}; \citeNP{fogw10}; \citeNP{chgc09}].
\end{itemize}

Transfer learning emphasizes the transfer of knowledge across related domains, tasks and distributions that are similar but not the same. 
The domain from which the knowledge is transferred is called the ``source'' domain and the domain to which the knowledge is transferred is 
called the ``target'' domain. In transfer learning scenarios, the source and 
target distributions are somewhat different, as they represent (potentially) related but not identical tasks.
The literature on transfer learning is fairly rich and varied (\textit{e.g.}, see [\citeNP{paya10}; \citeNP{sibe08}] and references therein), 
with much work done in the past 15 years \cite{thpr97}. The tasks may be learnt simultaneously \cite{caru97} 
or sequentially \cite{bogh00}. 

The novelty of our approach lies in the utilization of the theory of 
both classifier and cluster ensembles to address the challenge when there is very few labeled examples from the target class.
There are certain application domains such as the problem of land-cover classification
of spatially separated regions, where the setting is appropriate. 
Moreover, one does not always need to know \textit{a priori} whether the target is similar to the source domain. 
Though there is a recent paper that uses a single clustering to modify the weights of base classifiers in an ensemble 
in order to provide some transfer learning capability \cite{gafj08}, that algorithm is completely different from 
ours. 

Semi-supervised learning is a domain of machine learning where both labeled and unlabeled data are used to train a model -- typically with lot of unlabeled data and only a small amount 
of labeled data (see [\citeNP{bedl06}; \citeNP{zhgo09}] and the references therein for more details). There are several graph-based semi-supervised algorithms that use either 
the graph structure to spread labels from labeled to unlabeled samples, or optimize 
a loss function that includes a smoothness constraint derived from the graph [\citeNP{zhpd06}; \citeNP{subi09}; \citeNP{subi11}]. These approaches are typically non-parametric and transductive,
needing both the labeled and unlabeled data to be simultaneously available for the entire training process. 
\textbf{OAC\textsuperscript{3}} can use parametric classifiers so that old labeled data can be discarded once the classifier parameters are obtained,
leading to additional savings in speed and storage.

 A majority of previously proposed graph-based semi-supervised 
algorithms [\citeNP{zhgh02b}; \citeNP{joachims03}; \citeNP{bens05}; \citeNP{bedl06}] are based on minimizing squared-loss,  
while in \cite{subi11} (Measure Propagation -- \textbf{MP}), \cite{coja03} and \cite{tsuda05}, the authors used KL divergence. 
\textbf{OAC\textsuperscript{3}} uses certain Bregman divergences \cite{Censor-Zenios}, among which the KL divergence and squared loss constitute just a 
subset (further details are provided later, in Section \ref{sec:convergence}). This facilitates one to use
well-defined functions of measures for a specific problem in order to improve performance. 
Additionally, the techniques of variable splitting \cite{bopc11} and alternating minimization procedure \cite{beha02} are invoked to 
provide a more scalable solution.

The work that comes closest  
to ours is by Gao
\emph{et al}. \cite{galf09,Gao_TKDE}, which also combines the outputs of \textit{multiple}
supervised and unsupervised models. Here, it is assumed that each
model partitions the target dataset $\mathcal{X}$ into groups, so that the
instances in the same group share either the same predicted class label or the
same cluster label. The data, models and outputs are summarized by a bipartite
graph with connections only between group nodes and instance nodes.  A group node and
an instance node are connected if the instance is assigned to the group --- no
matter if it comes from a supervised or unsupervised model. The authors  cast the final consensus labeling  as an
optimization problem on this bipartite graph. To solve the optimization
problem, they introduce the Bipartite Graph-based Consensus Maximization
(\textbf {BGCM}) Algorithm, which is essentially a block coordinate descent
based algorithm that performs an iterative propagation of probability estimates
among neighboring nodes.
Note that their formulation requires {\em hard} classification and clustering inputs.  In contrast,  
\textbf{OAC\textsuperscript{3}}  essentially processes only two
fused models, namely an ensemble of classifiers and an ensemble of
clusterers, the constituents of both of which can be either hard or soft. 
Moreover, \textbf{OAC\textsuperscript{3}} avoids solving a difficult
correspondence problem --- \textit{i.e.}, aligning cluster labels to class
labels  --- implicitly tackled by \textbf {BGCM}, and has a lower computational complexity as well.

\section{Description of \textbf{OAC\textsuperscript{3}}}
\label{sec:C3E}

The proposed framework that combines classifiers and clusterers to generate a more consolidated classification is depicted in Fig. \ref{fig:1}. 
It is assumed that a set of classifiers (consisting of one or more classifiers) have been previously induced from a training set. Such classifiers could have been derived from labeled and unlabeled data, and they are part of the framework that will 
be used for classifying new data --- \textit{i.e.}, instances from the target set $\mathcal{X}=\{\mathbf{x}_{i}\}_{i=1}^{n}$. 
The target set is a test set that has not been used to build the classifiers. The classifiers are employed to estimate initial class probabilities for every instance $\mathbf{x}_{i}\in\mathcal{X}$. These probability distributions are stored as a set of vectors $\{\boldsymbol{\pi}_{i}\}_{i=1}^{n}$ and will be refined with the help of the clusterer(s). From this point of view, the clusterers provide supplementary constraints for classifying the instances of $\mathcal{X}$, 
with the rationale that similar instances are more likely to share the same class label. 

Given $k$ classes, denoted by $C=\{C_{\ell}\}_{\ell=1}^{k}$\footnote{C, with an overload of notation, is used here to denote a collection of classes and should not be confused with $C^{{k^{\prime}}}$ which 
is used to denote smoothness of a function.}, 
each of $\boldsymbol{\pi}_{i}$'s is of dimension $k$. In order to capture the similarities between the 
instances of $\mathcal{X}$, \textbf{OAC\textsuperscript{3}} also takes as input a similarity matrix \textbf{S}, which can be computed from a cluster ensemble, in such 
a way that each matrix entry corresponds to the relative co-occurrence of two instances in the same cluster \cite{stgh02b} --- considering all the data partitions that 
form the cluster ensemble induced from $\mathcal{X}$. Alternatively, \textbf{S} can be obtained from computing pair-wise similarities between instances, or   from a 
cophenetic matrix resulting from running a hierarchical clustering algorithm. To summarize, \textbf{OAC\textsuperscript{3}} receives as inputs a set of vectors 
$\{\boldsymbol{\pi}_{i}\}_{i=1}^{n}$ and a similarity matrix \textbf{S} for the target set. After processing these inputs, \textbf{OAC\textsuperscript{3}} outputs a 
consolidated classification --- represented by a set of vectors $\{\mathbf{y}_{i}\}_{i=1}^n \in \mathcal{S}\subseteq \mathbb{R}^{k}$, where $\mathbf{y}_{i}\propto\hat{P}(C\mid\mathbf{x}_{i})$ (estimated posterior class probability assignment)
--- for every instance in $\mathcal{X}$. This procedure is described in more detail in the sequel.

\begin{figure*}[ht]
 \centering
 \includegraphics[bb=0 0 902 489,scale=0.4]{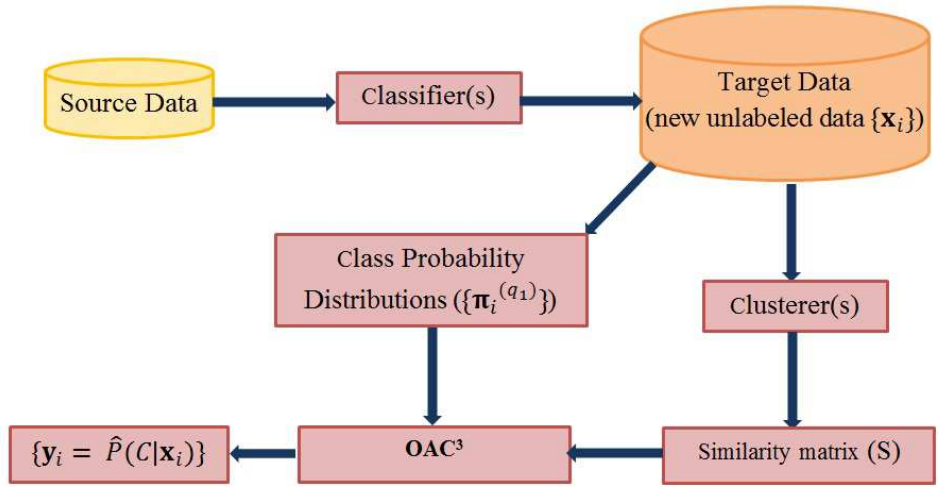}
 \caption{Overview of \textbf{OAC\textsuperscript{3}}.}
 \label{fig:1}
\end{figure*}

\subsection{Optimization Algorithm --- \textbf{OAC\textsuperscript{3}}}
\label{Algorithm}

Consider that $r_{1}$ ($r_{1}\geq1$) classifiers, indexed by $q_{1}$, and $r_{2}$ ($r_{2}\geq1$) clusterers, indexed by $q_{2}$, are employed to obtain a consolidated classification. The following steps (I-III) outline the proposed approach. Steps I and II can be seen as preliminary steps to get the inputs for \textbf{OAC\textsuperscript{3}}, while Step III is the optimization algorithm, which will be discussed in more detail.

{\bf Step I - Obtain input from classifiers.} The output of classifier $q_{1}$ for instance $\mathbf{x}_{i}$ is a $k$-dimensional class probability vector $\boldsymbol{\pi}_{i}^{(q_{1})}$. This probability vector denotes the probabilities for $\mathbf{x}_{i}$ being assigned to the corresponding classes (which might be soft or hard assignments).
From the set of such vectors $\{\boldsymbol{\pi}_{i}^{(q_{1})}\}_{q_{1}=1}^{r_{1}}$, an average vector can be computed for $\mathbf{x}_{i}$ as:

\begin{equation}
\label{eqn:t3}
\boldsymbol{\pi}_{i}=\displaystyle\frac{1}{r_{1}}\sum_{q_{1}=1}^{r_{1}}\boldsymbol{\pi}_{i}^{(q_{1})}.
\end{equation}

{\bf Step II - Obtain a similarity matrix.} A similarity matrix can be obtained in a number of ways, such as computing pair-wise similarities between instances from the original space of features. For high-dimensional data, it is usually more appropriate to use a cluster ensemble for computing similarities between instances of the target set. In this case, after applying $r_{2}$ clustering algorithms (clusterers) to $\mathcal{X}$, a similarity matrix $\mathbf{S}$ is computed. 
Assuming that each clustering is a hard data partition (possibly obtained from a particular subspace), the similarity between two instances is simply the fraction of the $r_{2}$ clustering solutions in which those two instances lie in the same cluster\footnote{A similarity matrix can also be defined for soft clusterings --- \textit{e.g.}, see \cite{pugh08b}.}. Note that such similarity matrices are byproducts of several cluster ensemble solutions, e.g., the \textbf{CSPA} algorithm in \cite {stgh02b}.


{\bf Step III - Obtain consolidated results from \textbf{OAC\textsuperscript{3}}}. Having defined the inputs for \textbf{OAC\textsuperscript{3}}, namely the set $\{\boldsymbol{\pi}_{i}\}_{i=1}^{n}$ and the similarity matrix, $\mathbf{S}$, the problem of combining classifiers and clusterers can be posed as an optimization problem whose objective is to minimize $J$ in (\ref{eqn:4}) with respect to the set of probability vectors $\{\mathbf{y}_{i}\}_{i=1}^{n}$, 
where $\mathbf{y}_{i}$ is the new and hopefully improved estimate of the aposteriori class probability distribution for a given instance in $\mathcal{X}$.\footnote{From now on, for generality, we assume that we have two ensembles (a classifier ensemble and a cluster ensemble), but note that 
each of these ensembles may be formed by a single component.}

\begin{equation}
\label{eqn:4}
J^{\text{original}}=\displaystyle\sum_{i\in\mathcal{X}}\mathcal{L}(\boldsymbol{\pi}_{i}, \mathbf{y}_{i})+\alpha \displaystyle\sum_{(i,j)\in\mathcal{X}}
s_{ij}\mathcal{L}(\mathbf{y}_{i}, \mathbf{y}_{j})
\end{equation}

The quantity $\mathcal{L}(\cdot,\cdot)$ denotes a loss function. Informally, the first term in Eq. (\ref{eqn:4}) captures dissimilarities between the class 
probabilities provided by the ensemble of classifiers and the output vectors $\{\mathbf{y}_{i}\}_{i=1}^{n}$. The second term encodes the cumulative weighted 
dissimilarity between all possible pairs $(\mathbf{y}_{i}, \mathbf{y}_{j})$. The weights to these pairs are assigned in proportion to the similarity values 
$s_{ij} \in[0,1]$ of matrix \textbf{S}. The coefficient $\alpha\in\mathbb{R_{+}}$ controls the relative importance of classifier and cluster ensembles. Therefore, 
minimizing the objective function over $\{\mathbf{y}_{i}\}_{i=1}^{n}$ involves combining the evidence provided by the ensembles in order 
to build a more consolidated classification.

The approach taken in this paper is quite general in the sense that any Bregman divergence that satisfies some specific 
properties (these properties will be introduced 
in more detail in section \ref{sec:convergence} where the discussion is more relevant)
can be used as a loss function $\mathcal{L}(\cdot,\cdot)$ in Eq. (\ref{eqn:4}).
So, before going into further details, the formal definition of Bregman divergence is provided.
\begin{definition}[\cite{breg67}, \cite{bamd05}] 
\label{def:1} 
Let $\phi:\mathcal{S}\rightarrow\mathbb{R},\mathcal{S}=\text{dom}(\phi)$ be a strictly
convex function defined on a convex set $\mathcal{S}\subseteq\mathbb{R}^{k}$ such that $\phi$ is differentiable on $\text{ri}(\mathcal{S})$, which is assumed to
be nonempty. The Bregman divergence $d_{\phi}:\mathcal{S}\times\text{ri}(\mathcal{S})\rightarrow[0,\infty)$ is defined as
$d_{\phi} (p, q) = \phi(p)-\phi(q)-\langle p-q, \grad_{\phi}(q)\rangle$,
where $\grad_{\phi}(q)$ represents the gradient vector of $\phi$ evaluated at $q$. 
\end{definition}

A specific Bregman Divergence (\textit{e.g.} KL-divergence) 
between two vectors $\mathbf{y}_{i}\text{ and }\mathbf{y}_{j}$ 
can be identified by a corresponding strictly convex function $\phi$ (\textit{e.g.} negative entropy for KL-divergence), and hence be written as $d_{\phi}(\mathbf{y}_{i}, \mathbf{y}_{j})$.
Following from Definition \ref{def:1}, $d_{\phi}(\mathbf{y}_{i}, \mathbf{y}_{j})\geq0\text{ }\forall\mathbf{y}_{i}\in\mathcal{S}, \mathbf{y}_{j}\in\text{ri}(\mathcal{S})$ and equality holds if and only if $\mathbf{y}_{i}=\mathbf{y}_{j}$. 
Using this notation, the objective function of \textbf{OAC\textsuperscript{3}}, that is going to be minimized over $\{\mathbf{y}_{i}\}_{i=1}^{n}$, can be rewritten as: 

\begin{equation}
\label{eqn:5}
J_{0} = \left[\displaystyle\sum_{i\in\mathcal{X}}d_{\phi}(\boldsymbol{\pi}_{i}, \mathbf{y}_{i})+\alpha\displaystyle\sum_{(i,j)\in\mathcal{X}} s_{ij}d_{\phi}(\mathbf{y}_{i}, \mathbf{y}_{j})\right].
\end{equation}

All Bregman divergences have the remarkable property that the single best (in terms of minimizing the net loss) representative of 
a set of vectors, is simply the expectation of this set (!) provided the divergence is computed with this representative as the second 
argument of $d_{\phi}(\cdot, \cdot)$ --- see Theorem \ref{mainthmbanerjee} in the sequel for a more formal statement of this result. 
Unfortunately, this simple form of the optimal solution is not valid if the variable to be optimized occurs as the first argument. 
In that case, however, one can work in the (Legendre) dual space, where the optimal solution has a simple form --- see \cite{bamd05} 
for details. Re-examining Eq. (\ref{eqn:5}), we notice that the ${\mathbf{y}_{i}}$'s to be minimized over occur both as first and second 
arguments of a Bregman divergence.
Hence optimization over ${\{\mathbf{y}_{i}\}_{i=1}^{n}}$ is not available in closed form. 
We circumvent this problem by creating two copies for each ${\mathbf{y}_{i}}$ --- the 
left copy, $\mathbf{y}^{(l)}_{i}$, and the right copy, $\mathbf{y}^{(r)}_{i}$. The left (right) copies are used whenever the variables are encountered in the first 
(second) argument of the Bregman divergences. 
In what follows, it will be clear that the right and left copies are updated iteratively, and an additional soft constraint 
is used to ensure that the two copies of a variable remain ``close 
enough'' during the updates. With this modification, we propose minimizing the following objective $J:\mathcal{S}^{n}\times\mathcal{S}^{n}\to [0, \infty)$:

\begin{equation}
\label{eqn:mainobj}
 J(\mathbf{y}^{(l)}, \mathbf{y}^{(r)}) 
= \left[\displaystyle\sum_{i=1}^{n}d_{\phi}(\boldsymbol{\pi}_{i}, \mathbf{y}_{i}^{(r)})
+ \alpha\displaystyle\sum_{i,j=1}^{n} s_{ij}d_{\phi}(\mathbf{y}_{i}^{(l)}, \mathbf{y}_{j}^{(r)})
+ \lambda\displaystyle\sum_{i=1}^{n}d_{\phi}(\mathbf{y}_{i}^{(l)}, \mathbf{y}_{i}^{(r)})\right],
\end{equation}
where, $\mathbf{y}^{(l)} = \left(\mathbf{y}_{i}^{(l)}\right)_{i=1}^{n}\in\mathcal{S}^{n}$
and $\mathbf{y}^{(r)} = \left(\mathbf{y}_{i}^{(r)}\right)_{i=1}^{n}\in\mathcal{S}^{n}$.

To solve the optimization problem in an efficient way, we first keep $\{\mathbf{y}^{(l)}_{i}\}_{i=1}^n$ 
and $\{\mathbf{y}^{(r)}_{i}\}_{i=1}^n\setminus\{\mathbf{y}^{(r)}_{j}\}$ fixed, 
and minimize the objective w.r.t. $\mathbf{y}^{(r)}_{j}$ only. The problem can, therefore, be written as:

\begin{equation}
\label{eqn:7} 
\min_{\mathbf{y}^{(r)}_{j}}\Biggl[d_{\phi}(\boldsymbol{\pi}^{(r)}_{j}, \mathbf{y}^{(r)}_{j}) +\alpha\displaystyle\sum_{i^{(l)}\in{\mathcal{X}}} s_{i^{(l)}j^{(r)}}d_{\phi}(\mathbf{y}^{(l)}_{i}, 
\mathbf{y}^{(r)}_{j})+\lambda_{j}^{(r)}d_{\phi}(\mathbf{y}^{(l)}_{j}, \mathbf{y}^{(r)}_{j})\Biggr],
\end{equation}

where $\lambda_{j}^{(r)}$ is the corresponding penalty parameter that is used to keep $\mathbf{y}^{(r)}_{j}$ and 
$\mathbf{y}^{(l)}_{j}$ close to each other.
For every valid assignment of $\{\mathbf{y}^{(l)}_{i}\}_{i=1}^{n}$, it can be shown that there is a unique 
minimizer ${\mathbf{y}^{(r)}_{j}}^{\ast}$ for the optimization problem in (\ref{eqn:7}). For that purpose,
a new Corollary is developed from the results of Theorem \ref{mainthmbanerjee} \cite{bamd05} that is stated below. 

\begin{theorem}[\cite{bamd05}]
\label{mainthmbanerjee}
Let $Y$ be a random variable that takes values in $\mathcal{Y}=\{\mathbf{y}_{i}\}_{i=1}^{n}\subset\mathcal{S}\subseteq \mathbb{R}^{k}$ following a probability measure $v$ such that $\mathbb{E}_{v}[Y]\in\text{ri}(\mathcal{S})$. 
 Given a Bregman divergence $d_\phi\colon\mathcal{S}\times\text{ri}(\mathcal{S})
 \rightarrow [0,\infty)$, the optimization problem $\min_{\mathbf{s}\in\text{ri}(\mathcal{S})}\mathbb{E}_{v}[d_\phi(Y,\mathbf{s})]$ 
 has a unique minimizer given by $\mathbf{s}^{\ast}=\boldsymbol{\mu}=\mathbb{E}_{v}[Y]$. 
\end{theorem}
   
To solve the problem formulated in Eq.~(\ref{eqn:7}), the following corollary is required:

\begin{corollary}
\label{corofrommainthm}
Let $\{Y_{i}\}_{i=1}^{n}$ be a set of random variables, each of which takes values in $\mathcal{Y}_{i}=\{\mathbf{y}_{ij}\}_{j=1}^{n_{i}}\subset \mathcal{S}\subseteq \mathbb{R}^{k}$ following a probability measure
$v_{i}$ such that $\mathbb{E}_{v_{i}}[Y_{i}]\in\text{ri}(\mathcal{S})$. Consider a Bregman divergence 
$d_\phi\colon\mathcal{S}\times\text{ri}(\mathcal{S})\rightarrow [0,\infty)$ and an objective function of the form $J_{\phi}(\mathbf{s})=\displaystyle\sum_{i=1}^{m}\alpha_{i}\mathbb{E}_{v_{i}}[d_\phi(Y_{i},\mathbf{s})]$
with $\alpha_{i}\in\mathbb{R_{+}}\text{ }\forall i$. This objective function has a unique minimizer given by 
$\mathbf{s}^{\ast}=\boldsymbol{\mu}=\left[\displaystyle\sum_{i=1}^{m}\alpha_{i}\mathbb{E}_{v_{i}}[Y_{i}]\right]/\left[\displaystyle\sum_{i=1}^{m}\alpha_{i}\right]$. 
\end{corollary}
    
\begin{proof}
Since $\mathbb{E}_{v_{i}}[Y_{i}]\in\text{ri}(\mathcal{S})\text{ }\forall i$, their convex combination should also belong to $\text{ri}(\mathcal{S})$,
implying that $\vect{\mu}\in\text{ri}(\mathcal{S})$. Now $\forall s\in\text{ri}(\mathcal{S})$ we have:
\begin{equation*}
J_{\phi}(\vect{s})-J_{\phi}(\vect{\mu}) = \displaystyle\sum_{i=1}^{m}\alpha_{i}\mathbb{E}_{v_{i}}[d_{\phi}(Y_{i},\vect{s})] - 
\displaystyle\sum_{i=1}^{m}\alpha_{i}\mathbb{E}_{v_{i}}[d_{\phi}(Y_{i},\vect{\mu})]
\end{equation*}
%
\begin{eqnarray*}
=\displaystyle\sum_{i=1}^{m}\alpha_{i}\left[\phi(\vect{\mu})-\phi(\vect{s})\right]
&-&\displaystyle\sum_{i=1}^{m}\alpha_{i}\big\langle\displaystyle\sum_{j=1}^{n}v_{ij}y_{ij}-\vect{s},\grad_{\phi}(\vect{s})\big\rangle\nonumber\\ 
&+&\displaystyle\sum_{i=1}^{m}\alpha_{i}\big\langle\displaystyle\sum_{j=1}^{n}v_{ij}y_{ij}-\vect{\mu},\grad_{\phi}(\vect{\mu})\big\rangle\nonumber\\
\end{eqnarray*}
\begin{equation*}
=\displaystyle\sum_{i=1}^{m}\alpha_{i}\left[\phi(\vect{\mu})-\phi(\vect{s})-\langle\vect{\mu}-\vect{s},\grad_{\phi}(\vect{s})\rangle\right]
=d_{\phi}(\vect{\mu},\vect{s})\displaystyle\sum_{i=1}^{m}\alpha_{i}\geq0
\end{equation*}
with equality only when $\vect{s}=\vect{\mu}$ following the strict convexity of $\phi$. Hence, $\vect{\mu}$ is the unique minimizer of the objective function $J_{\phi}$.
\end{proof}

From the results of Corollary \ref{corofrommainthm}, the unique minimizer of the optimization problem in (\ref{eqn:7}) is obtained as: 

\begin{equation}
\label{eqn:8} 
{\mathbf{y}^{(r)}_{j}}^{\ast} 
=\frac{\boldsymbol{\pi}_{j}^{(r)}+\gamma_{j}^{(r)}\displaystyle\sum_{i^{(l)}\in\mathcal{X}}\delta_{i^{(l)}j^{(r)}}\mathbf{y}^{(l)}_{i}
+\lambda_{j}^{(r)}\mathbf{y}^{(l)}_{j}}{1+\gamma_{j}^{(r)}+\lambda_{j}^{(r)}},
\end{equation}

where $\gamma_{j}^{(r)}=\alpha\sum_{i^{(l)}\in\mathcal{X}}s_{i^{(l)}j^{(r)}}$ and 
$\delta_{i^{(l)}j^{(r)}}=s_{i^{(l)}j^{(r)}}/\left[\sum_{i^{(l)}\in\mathcal{X}}s_{i^{(l)}j^{(r)}}\right]$. The same  
optimization in (\ref{eqn:7}) is repeated over all the $\mathbf{y}^{(r)}_{j}$'s. 
After the right copies are updated, the objective function is (sequentially) optimized with respect to all the $\mathbf{y}^{(l)}_{i}$'s. Like in the first step, $\{\mathbf{y}^{(l)}_{j}\}_{j=1}^n\setminus\{\mathbf{y}^{(l)}_{i}\}$ and $\{\mathbf{y}^{(r)}_{j}\}_{j=1}^n$ are kept fixed, and the difference between 
the left and right copies of $\mathbf{y}_{i}$ is penalized, so that the optimization with respect to $\mathbf{y}^{(l)}_{i}$ can be rewritten as:  

\begin{equation}
\label{eqn:9} 
\min_{\mathbf{y}^{(l)}_{i}}
\left[\alpha\displaystyle\sum_{j^{(r)}\in{\mathcal{X}}} s_{i^{(l)}j^{(r)}}d_{\phi}(\mathbf{y}^{(l)}_{i}, \mathbf{y}^{(r)}_{j})+\lambda_{i}^{(l)}d_{\phi}(\mathbf{y}^{(l)}_{i}, \mathbf{y}^{(r)}_{i})\right],
\end{equation}

where $\lambda_{i}^{(l)}$ is the corresponding penalty parameter. As mentioned earlier, one needs to work in the 
dual space now, using the convex function $\psi$ (Legendre dual of $\phi$) which is defined as:

\begin{equation}
\label{def:psi}
\psi(\mathbf{y}_{i})=\langle \mathbf{y}_{i}, \grad_{\phi}^{-1}(\mathbf{y}_{i})\rangle - 
\phi(\grad_{\phi}^{-1}(\mathbf{y}_{i})). 
\end{equation}

One can show that $\forall\mathbf{y}_{i}, \mathbf{y}_{j}\in \text{int}(\text{dom}(\phi))$,
$d_\phi(\mathbf{y}_{i},\mathbf{y}_{j})=d_\psi(\grad_{\phi}(\mathbf{y}_{j}),\grad_{\phi}(\mathbf{y}_{i}))$ 
--- see \cite{bamd05} for more details.
Thus, the optimization problem in (\ref{eqn:9}) can be rewritten in terms of the Bregman divergence associated with $\psi$ as follows:

\begin{eqnarray}
\label{eqn:t2} 
\min_{\grad_{\phi}(\mathbf{y}^{(l)}_{i})}
\Bigg[\alpha\displaystyle\sum_{j^{(r)}\in{\mathcal{X}}} s_{i^{(l)}j^{(r)}} d_{\psi}(\grad_{\phi}(\mathbf{y}^{(r)}_{j}),\grad_{\phi}(\mathbf{y}^{(l)}_{i}))
+\lambda_{i}^{(l)}d_{\psi}(\grad_{\phi}(\mathbf{y}^{(r)}_{i}), \grad_{\phi}(\mathbf{y}^{(l)}_{i}))\Bigg].
\end{eqnarray}

The unique minimizer of the problem in (\ref{eqn:t2}) can be computed using Corollary \ref{corofrommainthm}.  
$\grad_{\phi}$ is monotonic and invertible for $\phi$
being strictly convex and hence the inverse of the unique minimizer for the problem in (\ref{eqn:t2}) is also unique and 
equals to the unique minimizer for the problem in (\ref{eqn:9}). Therefore, the unique minimizer of the problem in (\ref{eqn:9}) 
with respect to ${\mathbf{y}^{(l)}_{i}}$ is given by:
\begin{equation}
\label{eqn:10} 
{\mathbf{y}^{(l)}_{i}}^{\ast} =
\grad_{\phi}^{-1}\left[
\frac{\gamma_{i}^{(l)}\displaystyle\sum_{j^{(r)}\in\mathcal{X}}\delta_{i^{(l)}j^{(r)}}\grad_{\phi}(\mathbf{y}^{(r)}_{j})
+\lambda_{i}^{(l)}\grad_{\phi}(\mathbf{y}^{(r)}_{i})}{\gamma_{i}^{(l)}+\lambda_{i}^{(l)}}\right],
\end{equation}
where $\gamma_{i}^{(l)}=\alpha\sum_{j^{(r)}\in\mathcal{X}}s_{i^{(l)}j^{(r)}}$ and 
$\delta_{i^{(l)}j^{(r)}}=s_{i^{(l)}j^{(r)}}/\left[\sum_{j^{(r)}\in\mathcal{X}}s_{i^{(l)}j^{(r)}}\right]$. 
For the experiments reported in this paper, the generalized I-divergence, defined as:
\begin{equation}
\label{eqn:dphi} 
d_{\phi}(\mathbf{y}_{i},\mathbf{y}_{j})=\displaystyle\sum_{\ell=1}^{k}y_{i\ell}\text{log}\left(\frac{y_{i\ell}}{y_{j\ell}}\right)-\displaystyle\sum_{\ell=1}^{k}(y_{i\ell}-y_{j\ell}), \forall \mathbf{y}_{i},\mathbf{y}_{j}\in \mathbb{R}^{k}_{+},
\end{equation}
has been used. The underlying convex function is then given by $\phi(\mathbf{y}_{i})=\displaystyle\sum_{\ell=1}^{k}y_{i\ell}\text{log}(y_{i\ell})$ so that 
$\grad_{\phi}(\mathbf{y}_{i})=\left(1+\text{log}({y}_{i\ell})\right)_{\ell=1}^{k}$. 
Thus, Eq. (\ref{eqn:10}) can be rewritten as:
\begin{equation}
\label{eqn:11} 
{\mathbf{y}^{(l)}_{i}}^{\ast,I}=\textit{exp}\left(\frac{\gamma_{i}^{(l)}\displaystyle\sum_{j^{(r)}\in\mathcal{X}}\delta_{i^{(l)}j^{(r)}}\grad_{\phi}(\mathbf{y}^{(r)}_{j})+\lambda_{i}^{(l)}\grad_{\phi}(\mathbf{y}^{(r)}_{i})}
{\gamma_{i}^{(l)}+\lambda_{i}^{(l)}}
\right)-\mathbf{1},
\end{equation}
where part of the superscript ``$,I$'' indicates that the optimal value corresponds to I-divergence.
Optimization over the left and right arguments of all the instances constitutes one pass (iteration) of the algorithm, and 
these two steps are repeated till convergence (a detailed proof for convergence will be given in Section \ref{sec:convergence}).
Upon convergence, all the $\mathbf{y}_{i}$'s are normalized to unit $L_{1}$ norm after averaging over the respective left and right copies, 
to yield the individual class probability distributions for every instance $\mathbf{x}_{i}\in\mathcal{X}$. The main steps of \textbf{OAC\textsuperscript{3}} are summarized in Algorithm ~\ref{algo:C3E}. 

\begin{algorithm}[tb]
\caption{--- \textbf{OAC\textsuperscript{3}}}
\label{algo:C3E}
 \textbf{Inputs}: $\{\vect{\pi}_{i}\}, \mathbf{S}$. \textbf{Output}: $\{\vect{y}_{i}\}$.\\
 \textbf{Step 0:} Initialize $\{\vect{y}_{i}^{(r)}\}, \{\vect{y}_{i}^{(l)}\}$ so that $\vect{y}_{i\ell}^{(r)}=\vect{y}_{i\ell}^{(l)}=\frac{1}{k}$ $\forall i\in\{1,2, \cdots,n\}$, $\forall \ell\in\{1,2,\cdots,k\}$.\\
  Loop until convergence:\\
  \textbf{Step 1:} Update $\vect{y}_{j}^{(r)}$ using Eq. \eqref{eqn:8} $\forall j\in\{1,2,\cdots,n\}$.\\
  \textbf{Step 2:} Update $\vect{y}_{i}^{(l)}$ using Eq. \eqref{eqn:10} $\forall i\in\{1,2,\cdots,n\}$.\\
  End Loop\\
  \textbf{Step 3:} Compute $\vect{y}_{i}=0.5[\vect{y}_{i}^{(l)}+\vect{y}_{i}^{(r)}]$ $\forall i\in\{1,2,\cdots,n\}$.\\ 
  \textbf{Step 4:} Normalize $\vect{y}_{i}$ $\forall i\in\{1,2,\cdots,n\}$.\\
\end{algorithm}

The update procedure captured by Eq. (\ref{eqn:10}) deserves some special attention. Depending on the divergence used, the update might 
not ensure that the left copies returned are in the correct domain. For example, if KL divergence is used, Eq. 
(\ref{eqn:10}) will not necessarily produce probabilities. In that case, one needs to use another Lagrangian multiplier to 
make sure that the returned values lie on simplex as has been done in \cite{subi11}.

\subsection{Time Complexity Analysis of OAC\textsuperscript{3}}
\label{TimeComplexity}

Considering that a trained ensemble of classifiers is available, the computation of the set of vectors $\{\boldsymbol{\pi}_{i}\}_{i=1}^{n}$ requires 
$O(nr_{1}k)$, where $n$ is the number of instances in the target set, $r_{1}$ is the number of components of the classifier ensemble, and $k$ is the number of class labels. 
Computing the similarity matrix, $\mathbf{S}$, is $O(r_{2}n^{2})$, where $r_{2}$ is the number of components of the cluster ensemble. Finally, 
having $\{\boldsymbol{\pi}_{i}\}_{i=1}^{n}$ and $\mathbf{S}$ available, the computational cost (per iteration) of  \textbf{OAC\textsuperscript{3}} is 
$O(kn^{2})$.
Actually, the computational bottleneck of \textbf{OAC\textsuperscript{3}} is not the optimization algorithm itself, whose main steps 
(1 and 2) can be parallelized (this can be identified by a careful inspection of Eq. (\ref{eqn:8}) and (\ref{eqn:10})), but the computation of the similarity matrix. Note that low values in the similarity matrix can 
often be zeroed out to further speed up the computation, without having much impact on the results.
\section{Convergence Analysis of OAC\textsuperscript{3}}
\label{sec:convergence}

We claim that \textbf{OAC\textsuperscript{3}} makes the objective $J$ in Eq. \ref{eqn:mainobj} converge to some unique 
minimizer when Bregman divergences with the following properties are used as loss functions:
\begin{enumerate}
 \item $d_{\phi} (\mathbf{p}, \mathbf{q})$ is strictly convex in $\mathbf{p}$ and $\mathbf{q}$ separately. 
 \item $d_{\phi} (\mathbf{p}, \mathbf{q})$ is jointly convex w.r.t $\mathbf{p}$ and $\mathbf{q}$. 
 \item The level sets $\{\mathbf{q}:d_{\phi} (\mathbf{p}, \mathbf{q})\le r\}$ are bounded for any given $\mathbf{p}\in\mathcal{S}$.
 \item $d_{\phi} (\mathbf{p}, \mathbf{q})$ is lower-semi-continuous in $\mathbf{p}$ and $\mathbf{q}$ jointly.
 \item If $d_{\phi} (\mathbf{p}^{t}, \mathbf{q}^{t})\to 0$ and $\mathbf{p}^{t}$ or $\mathbf{q}^{t}$ is bounded, then $\mathbf{p}^{t}\to \mathbf{q}^{t}$ and $\mathbf{q}^{t}\to \mathbf{p}^{t}$.
 \item If $\mathbf{p}\in\mathcal{S}$ and $\mathbf{q}^{t}\to \mathbf{p}$, then $d_{\phi} (\mathbf{p}, \mathbf{q}^{t})\to 0$.
\end{enumerate}

Bregman divergences that satisfy the above properties include a large number of useful loss functions 
such as the well-known squared loss, KL-divergence, generalized I-divergence, logistic loss, 
Itakura-Saito distance and  Bose-Einstein entropy \cite{wada03}. These divergences along with their associated
strictly convex functions $\phi(.)$ and domains are listed in Table \ref{table:BregmanDiv}. 

\begin{table}
\tbl{Examples of Bregman divergences that satisfy properties (a) to (f)}
{
\begin{tabular}{|| c | c | c | c ||}
\hline
\hline
  Domain & $\phi(\mathbf{p})$ & $d_{\phi}(\mathbf{p},\mathbf{q})$ & Divergence \\\hline
  $\mathbb{R}$ &  $p^{2}$ & $(p-q)^2$ & Squared Loss\\ \hline
  $[0,1]$ &  $p\text{log}(p) + (1-p)\text{log}(1-p)$ & $p\text{log}\big(\frac{p}{q}\big) + (1-p)\text{log}\big(\frac{1-p}{1-q}\big)$ & Logistic Loss\\ \hline
  $\mathbb{R}_{+}$ &  $p\text{log}(p) - (1+p)\text{log}(1+p)$ & $p\text{log}\big(\frac{p}{q}\big) - (1+p)\text{log}\big(\frac{1+p}{1+q}\big)$ & Bose-Einstein Entropy\\ \hline
  $\mathbb{R}_{++}$ &  $-\text{log}(p)$& $\frac{p}{q} - \text{log}\big(\frac{p}{q}\big) - 1$ & Itakura-Saito Distance\\ \hline
  $\mathbb{R}^{k}$ &  $||p||^{2}$ & $||p-q||^{2}$ & Squared Euclidean Distance\\ \hline
  $k$-simplex & $\displaystyle\sum_{i=1}^{k}p_{i}\text{log}_{2}(p_{i})$ & $\displaystyle\sum_{i=1}^{k}
p_{i}\text{log}_{2}\big(\frac{p_{i}}{q_{i}}\big)$ & KL-Divergence\\ \hline
  $\mathbb{R}_{+}^{k}$ & $\displaystyle\sum_{i=1}^{k}p_{i}\text{log}(p_{i})$ & $\displaystyle\sum_{i=1}^{k}
p_{i}\text{log}\big(\frac{p_{i}}{q_{i}}\big) - \displaystyle\sum_{i=1}^{k} (p_{i}-q_{i})$ & Generalized I-Divergence\\ \hline
\hline
\end{tabular}}
\label{table:BregmanDiv}
\end{table}
 
An alternating optimization algorithm, in general, is not guaranteed to converge. Even if it converges it might not converge to 
the locally optimal solution. Some authors [\citeNP{chgo59}; \citeNP{zang69}; \citeNP{wu82}; \citeNP{beha03}] have shown that the convergence guarantee of alternating optimization 
can be analyzed using the topological properties of the objective and the space over which it is optimized. 
Others have used information geometry [\citeNP{cstu84}; \citeNP{wasc03}; \citeNP{subi11}] to analyze the 
convergence as well as a combination of both information geometry and topological properties of the objective \cite{guby05}. In this paper, the information geometry approach  
is utilized to show that the proposed optimization procedure converges to the global minima of the objective $J$ in \ref{eqn:mainobj}. 

At this point it is worth mentioning the connection of the optimization framework with other related approaches.
The algorithms in [\citeNP{zhgh02b}; \citeNP{bens05}] are based on minimizing squared-loss
and are only suitable for binary classification problems. Multi-class extension of these algorithms is entirely based on one-vs-all strategy. 
\textbf{MP} \cite{subi11}, on the other hand, is suitable for multi-class 
problems and additionally provides guard against degenerate solutions (those that assign equal confidence to all classes). 
\textbf{OAC\textsuperscript{3}} does not guard against degenerate solutions but can easily be extended to 
alleviate the same problem with the addition of a single tuning parameter. In the experiments reported, no significant difference in performance is 
observed with this extension and hence it is discarded to help tune one less model parameter. 
Label Propagation (\cite{zh05} -- \textbf{LP}) is another related algorithm and has been shown to converge to the optimal 
solution. In \cite{subi11}, the authors also proved that their algorithm converges but the convergence rate (for KL divergence) is not proven and only empirical evidence is given for a 
linear rate. In this paper, apart from generalizing these algorithms with a larger class of Bregman divergences, we provide
proofs for linear rate of convergence for generalized I divergence and KL divergence (the proof for squared loss follows directly from the analysis of \cite{subi11}).
Spectral graph transduction \citeNP{joachims03} is an approximate solution to the NP-hard
norm-cut problem. However, this algorithm requires eigen-decomposition of a matrix of size $n\times n$, where $n$ is the number of instances, 
which is inefficient for very large data sets. Manifold regularization \cite{bens05} is a 
general framework in which a parametric loss function is defined over the labeled samples and is regularized by graph
smoothness term defined over both the labeled and unlabeled samples. In the algorithms proposed therein, 
one either needs to invert an $n\times n$ matrix or use optimization techniques for general SVM in case there is no closed 
form solution. Both \textbf{OAC\textsuperscript{3}} and \textbf{MP}, on the other hand, have closed form solutions corresponding to each update and hence are 
perfectly suitable for large scale applications. Information regularization \cite{coja03}, in essence, works on the same 
intuition as \textbf{OAC\textsuperscript{3}}, but does not provide any proof of convergence and one of the steps
of the optimization does not have a closed form solution -- a concern for large data applications.
\cite{tsuda05} extended the works of \cite{coja03} to hyper-graphs and used 
closed form solutions in both steps of the alternating minimization procedure which, surprisingly, can be seen as a special 
case of \textbf{MP}.

We now give a sketch of the proof of convergence of \textbf{OAC\textsuperscript{3}}. 
The so-called 5-points property (5-pp) of the objective function $J$ is essential to analyze the convergence. 
If $J$ satisfies the 3-points property (3-pp) and the 4-points property (4-pp), then it satisfies the 5-pp.
Therefore, to prove 5-pp of $J$, we will try to prove that it satisfies both 3-pp and 4-pp. However, this proof is not easy for 
any arbitrary Bregman divergence. In \cite{subi11}, the authors followed the procedure of \cite{cstu84} 
to prove the convergence of a slightly different objective that involves KL-divergence as a loss function. The proof there is specific to KL-divergence  
and does not generalize to Bregman divergences with properties (a) to (f). Therefore, we take a more subtle route in proving 
the 3-pp and 4-pp of $J$. We show that the objective function $J$, which is a sum of Bregman divergences of different pairs of variables,
can itself be thought of as a Bregman divergence in some joint space. This Bregman divergence also satisfies the properties (a) to (f), 
which then allows one to use the convergence tools developed by \cite{wada03}. The formal proof for convergence is placed in 
appendix \ref{app:convergence} to facilitate an easy perusal of the paper.  
\section{Analysis of Rate of Convergence for OAC\textsuperscript{3}}
\label{rateofconv}
In practical applications, the rate of convergence of any optimization algorithm is of great importance.
To analyze the same, we use some formulations that were derived in \cite{beha03} 
to characterize the local convergence rate of alternating 
minimization type of algorithms in general. In this section, we will first explain the tools and then show 
that the analysis applies to the objective function $J$ seamlessly.
The details of the tools are skipped here though and only the main lemmata and theorems are provided. 

\subsection{Tools for Analyzing Local Rate of Convergence}
Let us consider a variable $\mathbf{z}\in\mathcal{S}^{2n}$ where 
$\mathbf{z} = (\mathbf{z}_{n^{\prime}})_{n^{\prime}=1}^{2n}$ and $\mathbf{z}_{n^{\prime}}\in\mathcal{S}\text{ }\forall n^{\prime}$. 
Assume functions $M_{n^{\prime}}: \mathcal{S}^{2n-1}\to \mathcal{S}$ $\forall n^{\prime}$ which are defined as:
\begin{equation}
 M_{n^{\prime}}(\tilde{\mathbf{z}}_{n^{\prime}})  
=  \underset{\mathbf{z}_{n^{\prime}}\in \mathcal{S}}{\operatorname{argmin }} 
f(\mathbf{z}_{1},\cdots,\mathbf{z}_{n^{\prime}-1}, \mathbf{z}_{n^{\prime}}, \mathbf{z}_{n^{\prime}+1}, \cdots, \mathbf{z}_{2n})
\end{equation}
Here, $\tilde{\mathbf{z}}_{n^{\prime}} = (\mathbf{z}_{1},\cdots,\mathbf{z}_{n^{\prime}-1}, \mathbf{z}_{n^{\prime}+1}, \cdots, \mathbf{z}_{2n})$.
Corresponding to each $M_{n^{\prime}}$ we also define a function 
$C_{n^{\prime}}:\mathcal{S}^{2n}\to \mathcal{S}^{2n}$ as:
\begin{equation}
 C_{n^{\prime}}(\mathbf{z}_{1},\cdots,\mathbf{z}_{n^{\prime}-1}, \mathbf{z}_{n^{\prime}}, 
\mathbf{z}_{n^{\prime}+1}, \cdots, \mathbf{z}_{2n}) =
\left(\mathbf{z}_{1},\cdots,\mathbf{z}_{n^{\prime}-1}, M_{n^{\prime}}(\tilde{\mathbf{z}}_{n^{\prime}}), 
\mathbf{z}_{n^{\prime}+1}, \cdots, \mathbf{z}_{2n}\right) 
\end{equation}
Moreover, one complete execution of alternating minimization step can conveniently 
be represented by a function $\mathbb{S}:\mathcal{S}^{2n}\to \mathcal{S}^{2n}$:
\begin{equation}
\label{eqn:mainmap}
\mathbb{S}(\mathbf{z}) = C_{1}\circ C_{2}\circ \cdots C_{2n}(\mathbf{z}). 
\end{equation}

\begin{lemma}
\label{convlemma2}
Let $f:\mathcal{S}^{2n}\to \mathbb{R}$ satisfy the following conditions:
\begin{enumerate}
 \item $f$ is C$^{2}$ in a neighborhood of $\mathbf{z}^{*}$, $\mathbf{z}^{*}$ being a local minimizer of $f$; 
 \item $\grad^{2}f(\mathbf{z}^{*})$ is positive definite;
 \item There is a neighborhood $\mathcal{N}$ of $\mathbf{z}^{*}$ on which $f$ is strictly convex, and such that for $n^{\prime}\in\{1,2,\cdots,2n\}$ if 
$\mathbf{z} = \mathbf{z}_{\mathbf{z}_{n^{\prime}}}^{*}$ locally minimizes $g_{n^{\prime}}(\mathbf{z}_{\mathbf{z}_{n^{\prime}}}) = f(\mathbf{z}_{\mathbf{z}_{n^{\prime}}})$
with $\mathbf{z}_{\mathbf{z}_{n^{\prime}}}$ indicating that all variables except $\mathbf{z}_{n^{\prime}}$
are held fixed, then $\mathbf{z}_{\mathbf{z}_{n^{\prime}}}^{*}$ is also the unique global minimizer of $g_{n^{\prime}}(\mathbf{z}_{\mathbf{z}_{n^{\prime}}})$.
\end{enumerate}
Then in some neighborhood of $\mathbf{z}^{*}$, the minimizing function
$M_{n^{\prime}}$ exists and is continuously differentiable $\forall n^{\prime}\in\{1,2,3,\cdots,2n\}$.
\end{lemma}

\begin{lemma}
\label{convlemma3}
 Let $f:\mathcal{S}^{n}\to \mathbb{R}$ be differentiable and satisfy the conditions of Lemma \ref{convlemma2}.
Then $\rho(\grad_{\mathbb{S}}(\mathbf{z}^{*}))< 1$ where $\grad_{\mathbb{S}}(\mathbf{z}^{*})$ is the Jacobian of 
the mapping $\mathbb{S}$ evaluated at $\mathbf{z}^{*}$ and $\rho$ is the spectral radius of the Jacobian.
\end{lemma}

Before presenting the main theorem from \cite{beha03}, the formal definition of q-linear rate of convergence is provided below. 
The ``q'' in this definition stands for quotient.
\begin{definition}[q-linear rate of convergence]
 A sequence $\{\mathbf{z}^{(t)}\}\rightarrow \mathbf{z}^{\ast}$ q-linearly iff $\exists t_{0}\ge 0$ and $\exists \rho \in [0,1)$ such that 
$\forall t\ge t_{0}$, $||\mathbf{z}^{(t+1)} - \mathbf{z}^{\ast}||\le \rho ||\mathbf{z}^{(t)} - \mathbf{z}^{\ast}||$ 
\end{definition}

\begin{theorem}
\label{convthm}
 Let $\mathbf{z}^{*}$ be a local minimizer of $f:\mathcal{S}^{n}\to \mathbb{R}$ for 
which $\grad^{2}f(\mathbf{z}^{*})$ is positive definite and let $f$ be C$^{2}$ in a neighborhood of 
$\mathbf{z}^{*}$. Also let assumption (c) of Lemma \ref{convlemma3} hold for $\mathbf{z}^{*}$.
Then there is a neighborhood $\mathcal{N}$ of $\mathbf{z}^{*}$ such that for any
$\mathbf{z}^{(0)}\in\mathcal{N}$, the corresponding iteration sequence $\{\mathbf{z}^{(t+l)}=\mathbb{S}(\mathbf{z}^{(t)}):t=0,1,...\}$
converges q-linearly to $\mathbf{z}^{*}$.
\end{theorem}

\subsection{Hessian Calculation of $J$}
From the theorems and lemmata presented in the previous subsection, one can observe that the Hessian of the objective being positive definite is a 
critical condition. Therefore, we will try to show that $\grad^{2}J$ is positive definite for some of the Bregman divergences.
According to Eq. (\ref{eqn:mainobj}), $\grad J$ involves the following terms:
\begin{eqnarray*}
 \grad_{\mathbf{y}_{i}^{(l)}} J = \alpha\displaystyle\sum_{j=1;j\neq i}^{n}s_{ij}
\left[\grad_{\phi}(\mathbf{y}_{i}^{(l)}) - \grad_{\phi}(\mathbf{y}_{j}^{(r)})\right]
+ \lambda \left[\grad_{\phi}(\mathbf{y}_{i}^{(l)}) - \grad_{\phi}(\mathbf{y}_{i}^{(r)})\right]
\end{eqnarray*}

\begin{eqnarray*}
\grad_{\mathbf{y}_{j}^{(r)}} J = \left[(\grad_{\phi}(\mathbf{y}_{j}^{(r)}) - \grad_{\phi}(\boldsymbol{\pi}_{j}))
+ \alpha\displaystyle\sum_{j=1;j\neq i}^{n}\big(\grad_{\phi}(\mathbf{y}_{j}^{(r)}) - \grad_{\phi}(\mathbf{y}_{i}^{(l)})\big)\right.\\ \nonumber
\left. + \lambda \left[\grad_{\phi}(\mathbf{y}_{j}^{(r)}) - \grad_{\phi}(\mathbf{y}_{i}^{(l)})\right]\right]^{\dagger}\grad_{\phi}^{2}(\mathbf{y}_{j}^{(r)}).
\end{eqnarray*}

$\grad^{2}J$, derived from the above equations, has the following terms:
\begin{eqnarray*}
 \grad^{2}_{\mathbf{y}_{i}^{(l)}, \mathbf{y}_{i}^{(l)}} J = \big(\alpha\displaystyle\sum_{j=1;j\neq i}^{n}s_{ij} + \lambda \big)\grad_{\phi}^{2}(\mathbf{y}_{i}^{(l)})
\end{eqnarray*}
\begin{eqnarray*}
 \grad^{2}_{\mathbf{y}_{j}^{(r)}, \mathbf{y}_{j}^{(r)}} J = \left[\big(1+\alpha\displaystyle\sum_{i=1;j\neq i}^{n}s_{ij} + \lambda\big)\mathbf{y}_{j}^{(r)}
-\boldsymbol{\pi}_{j} -\alpha\displaystyle\sum_{i=1;j\neq i}^{n}s_{ij}\mathbf{y}_{i}^{(l)} -\lambda\mathbf{y}_{j}^{(l)}\right]^{\dagger}\grad_{\phi}^{3} (\mathbf{y}_{j}^{(r)})\\ \nonumber
+ \big(1+\alpha\displaystyle\sum_{i=1;j\neq i}^{n}s_{ij} + \lambda\big)\grad_{\phi}^{2} (\mathbf{y}_{j}^{(r)})
\end{eqnarray*}
\begin{eqnarray*}
 \grad^{2}_{\mathbf{y}_{j}^{(r)}, \mathbf{y}_{i}^{(l)}} J =  \grad^{2}_{\mathbf{y}_{i}^{(l)}, \mathbf{y}_{j}^{(r)}} J 
 = -\alpha s_{ij}\grad^{2}_{\phi}(\mathbf{y}_{j}^{(r)}) (i\neq j)
\end{eqnarray*}
\begin{eqnarray*}
 \grad^{2}_{\mathbf{y}_{i}^{(r)}, \mathbf{y}_{i}^{(l)}} J =  \grad^{2}_{\mathbf{y}_{i}^{(l)}, \mathbf{y}_{i}^{(r)}} J
 = -\lambda \grad^{2}_{\phi}(\mathbf{y}_{i}^{(r)})
\end{eqnarray*}
\begin{equation*}
 \grad^{2}_{\mathbf{y}_{i}^{(l)}, \mathbf{y}_{j}^{(l)}} J = \grad^{2}_{\mathbf{y}_{i}^{(r)}, \mathbf{y}_{j}^{(r)}} J = 0, (i\neq j) 
\end{equation*}
Note that this calculation is valid for any Bregman divergence within the assumed family.  

\subsection{Hessian Calculation for KL and Generalized I divergence}
We are now in a position to show that the Hessian of the objective $J$ is positive definite when KL or 
I-divergence is used as Bregman divergence. Recall from table \ref{table:BregmanDiv} that the generating functions $\phi(.)$'s for KL 
and I-divergence differ only by a linear term and hence the Hessian of the objective $J$ would be the same for these two cases. 
We list different terms of the Hessian here:
\begin{eqnarray}
\label{hessianstart}
 \grad^{2}_{\mathbf{y}_{i}^{(l)}, \mathbf{y}_{i}^{(l)}} J = \left(\alpha\displaystyle\sum_{j=1;j\neq i}^{n}s_{ij} + \lambda\right)
\text{diag}\big((1/\mathbf{y}_{i\ell}^{(l)})_{\ell=1}^{k}\big)
\end{eqnarray}
\begin{eqnarray}
 \grad^{2}_{\mathbf{y}_{j}^{(r)}, \mathbf{y}_{j}^{(r)}} J = \text{diag}\left(\left(\frac{\boldsymbol\pi_{j\ell} + 
\alpha \displaystyle\sum_{i=1;j\neq i}^{n} s_{ij}\mathbf{y}_{i\ell}^{(l)} + \lambda\mathbf{y}_{j\ell}^{(l)}}
{{\mathbf{y}_{j}^{(r\ell)}}^{2}}\right)_{\ell=1}^{k}\right)
\end{eqnarray}
\begin{eqnarray}
 \grad^{2}_{\mathbf{y}_{j}^{(r)}, \mathbf{y}_{i}^{(l)}} J =  \grad^{2}_{\mathbf{y}_{i}^{(l)}, \mathbf{y}_{j}^{(r)}} J 
 = -\alpha s_{ij}\text{diag}\big((1/\mathbf{y}_{j\ell}^{(r)})_{\ell=1}^{k}\big) (i\neq j) 
\end{eqnarray}
\begin{eqnarray}
 \grad^{2}_{\mathbf{y}_{i}^{(r)}, \mathbf{y}_{i}^{(l)}} J =  \grad^{2}_{\mathbf{y}_{i}^{(l)}, \mathbf{y}_{i}^{(r)}} J 
 = -\lambda \text{diag}\big((1/\mathbf{y}_{i\ell}^{(r)})_{\ell=1}^{k}\big)
\end{eqnarray}
\begin{equation}
\label{hessianend}
 \grad^{2}_{\mathbf{y}_{i}^{(l)}, \mathbf{y}_{j}^{(l)}} J = \grad^{2}_{\mathbf{y}_{i}^{(r)}, \mathbf{y}_{j}^{(r)}} J = 0, (i\neq j). 
\end{equation}

Using Eqs. (\ref{hessianstart}) to (\ref{hessianend}) and some simple algebra, the following lemma can be proved. 
\begin{lemma}
\label{convlemma1}
 $\mathcal{H} = \grad^{2}J$ is positive definite over the domain of $J$ under the assumption 
$\displaystyle\sum_{i=1}^{n}\displaystyle\sum_{\ell=1}^{k}\pi_{i\ell}> 0$ when KL or generalized I divergence is used as a 
Bregman divergence.
\end{lemma}
The proof is placed in Appendix \ref{app:rateconv}.

\subsection{Convergence Rate of OAC\textsuperscript{3} with KL and I-divergence}
From Lemma \ref{convlemma1}, we have $\mathcal{H}$ is positive definite if 
$\displaystyle\sum_{i=1}^{n}\displaystyle\sum_{\ell=1}^{k}\pi_{i\ell} > 0$. This is always the 
case as $\boldsymbol{\pi}_{i}$ represents some probability assignment. Also, if generalized I divergence or KL divergence is 
used as the Bregman divergence, $J\in C^{\infty}$ (\emph{i.e.} $J$ is a smooth function). From Lemma \ref{Jlemma}, we have that 
$J$ is jointly strictly 
convex and hence has a unique minimizer. From the same Lemma, $J$ 
is separately strictly convex w.r.t each of its arguments. Therefore, with other variables fixed at some value, 
$J$ has a unique minimizer w.r.t one particular variable. Hence, all the conditions 
mentioned in Lemma \ref{convlemma2} are satisfied for $J$ in its entire domain. 
Therefore, following Theorem \ref{convthm} we can conclude that $J$ converges globally 
(implying that $\mathcal{N} = \text{dom}(J)$) to its unique minimizer q-linearly using \textbf{OAC\textsuperscript{3}}.
Note that when the Bregman divergence is the squared Euclidean distance, variable splitting 
is not required at all. The updates involve only one set of copies (\emph{i.e.} there is no need to maintain
left and right copies) and the q-linear rate of convergence of the objective $J$ can be proved following the same method 
as done in \cite{subi11}. The proof uses Perron-Frobenius theorem to bound the maximum eigen-value of the transformation 
matrix used to update the values of the probability assignments. Thus, \textbf{OAC\textsuperscript{3}} converges q-linearly at least 
when squared Euclidean, KL or I divergence is used as loss function. One needs to compute the Hessian or use some other tricks for other 
Bregman divergences having properties (a) to (f).  
\section{Experimental Evaluation}
\label{sec:exp}

First we provide a simple pedagogical example that illustrates how the supplementary constraints provided by clustering algorithms can be useful for improving the generalization capability of classifiers. Section \ref{Sensitivity_Analysis} reports sensitivity analyses on the \textbf{OAC\textsuperscript{3}} parameters.
Then, in Section \ref{BGCMComparison}, we compare the performance of \textbf{OAC\textsuperscript{3}} with the recently proposed \textbf{BGCM} \cite{galf09,Gao_TKDE}. This comparison is straightforward and fair, since it uses the same datasets, as well as the same outputs of the base models, which were kindly provided by the authors of this paper. For a comparison with other 
semi-supervised methods, the design space  is much larger, since we are now faced with a variety of classification and clustering algorithms to choose from as the base models in \textbf{OAC\textsuperscript{3}}, as well as a variety of semi-supervised methods to 
 compare with. Given the space available, in Section \ref{SVMComparison} we use simple (linear) base methods, and pick the popular Semi-Supervised Linear Support Vector Machine (\textbf{S\textsuperscript{3}VM}) \cite{sind06} for comparison. Finally, in Section \ref{TransferLearning} we report empirical results for transfer learning settings.

\subsection{Pedagogical Example}
\label{PedagogicalExample}

\begin{figure*}[htbp]
\begin{minipage}[b]{0.5\linewidth}
\centering
 \includegraphics[bb=0 0 1012 626, scale=0.2]{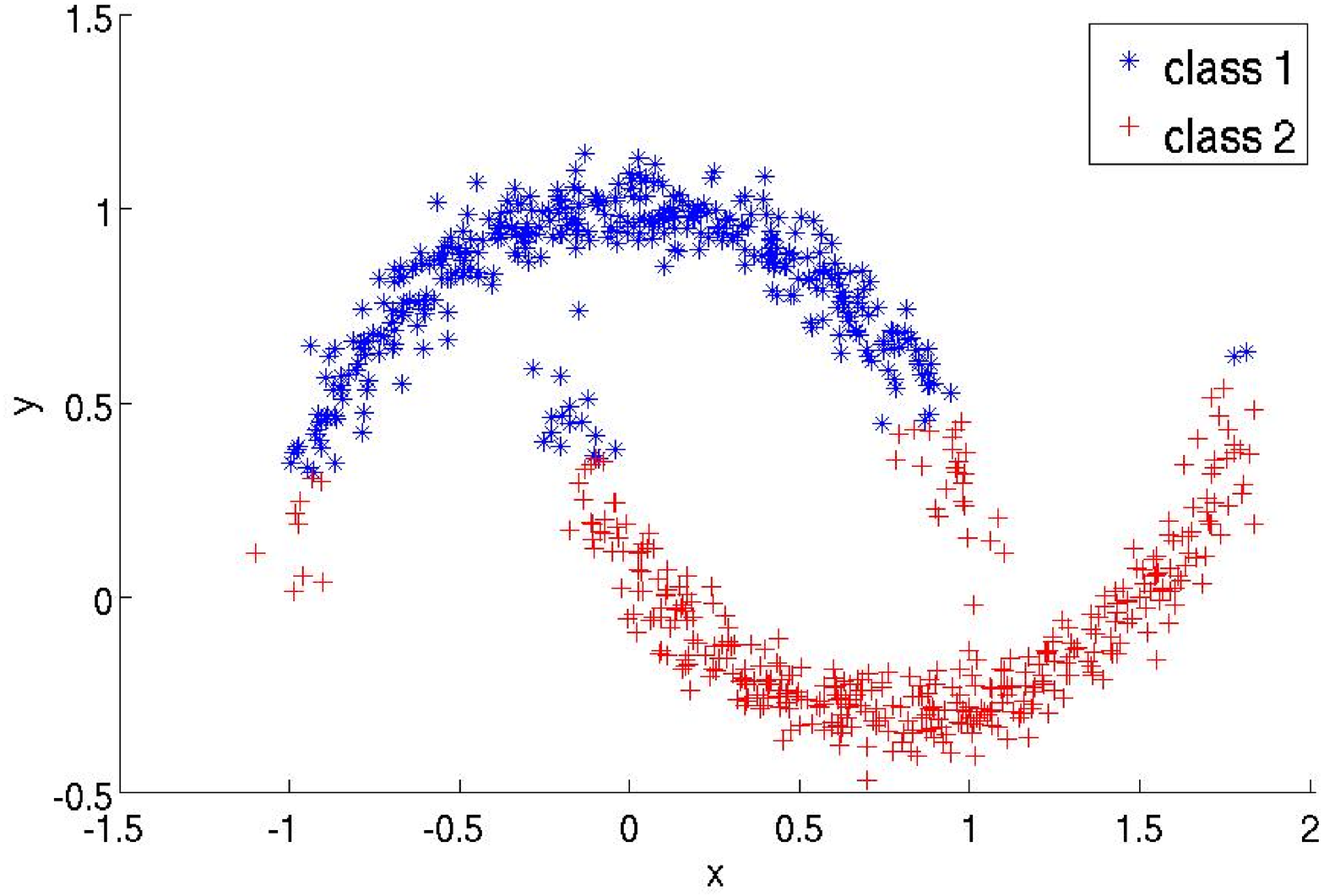}
 \caption{Class Labels from the Classifier Ensemble.}
\label{hm_ensemble}
\end{minipage}
\hspace{0.5cm}
\begin{minipage}[b]{0.5\linewidth}
\centering
 \includegraphics[bb=0 0 1012 626, scale=0.2]{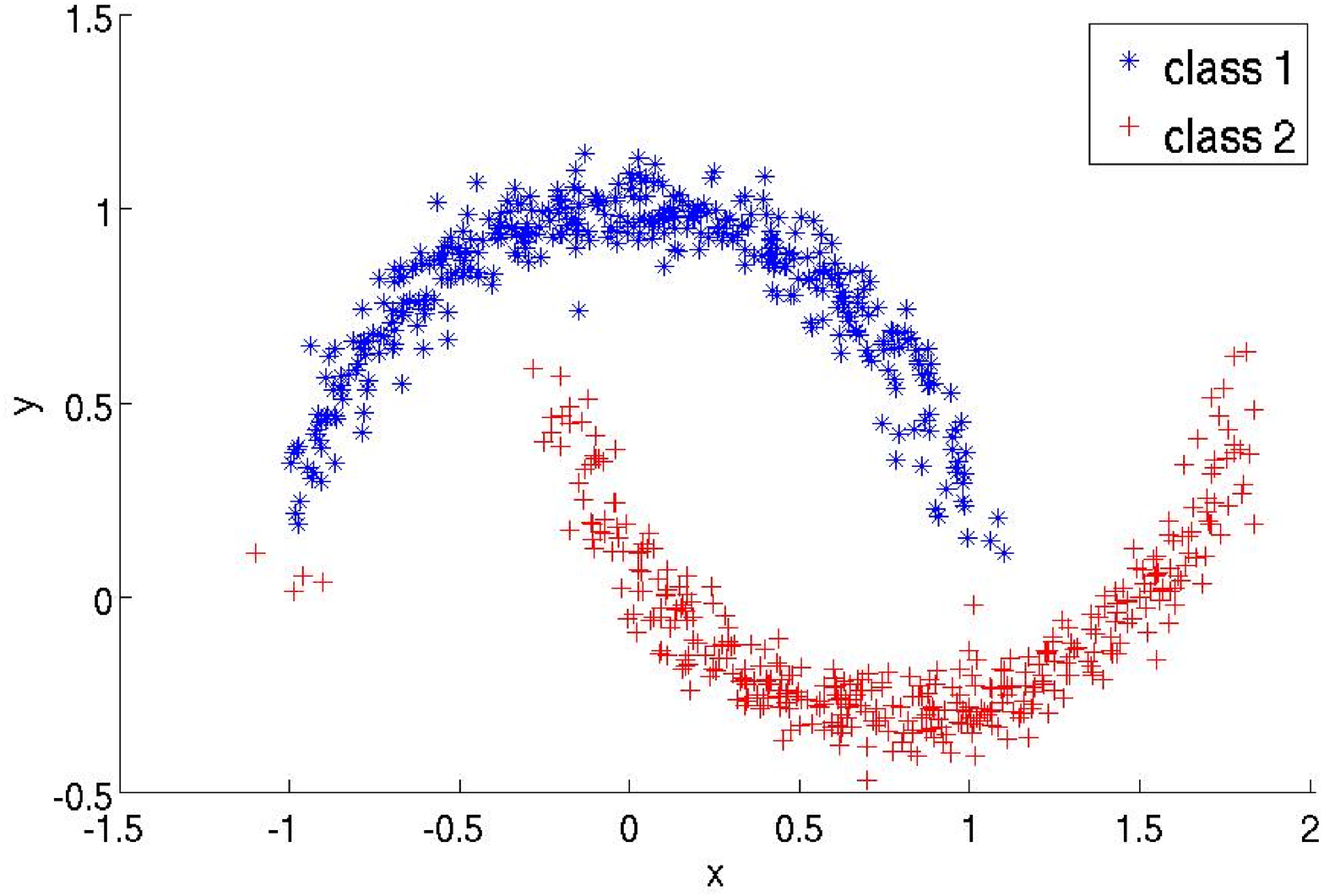}
 \caption{Class Labels from \textbf{OAC\textsuperscript{3}}.}
\label{C3E_hm}
\end{minipage}
\end{figure*}

Consider the two-dimensional dataset known as \textit{Half-Moon}, which
has two classes, each of which represented by 400 instances. From this dataset, 2\% of the instances are used for training,
whereas the remaining instances are used for testing (target set). 
A classifier ensemble formed by three well-known classifiers (Decision Tree, Linear Discriminant, and Generalized Logistic Regression) 
are adopted. In order to get a cluster ensemble, a single linkage (hierarchical) clustering algorithm is chosen. The cluster ensemble is then obtained from five data partitions represented in the dendrogram, which is cut for different number of clusters (from 4 to 8). 
Fig. \ref{hm_ensemble} shows the target data class labels obtained from the standalone use of the classifier ensemble, whereas Fig. \ref{C3E_hm} shows the corresponding results achieved by \textbf{OAC\textsuperscript{3}}. The parameter values were set by using cross-validation. 
In particular, we set $\alpha= 0.0001$ and $\lambda_{i}^{(r)}=\lambda_{i}^{(l)}=\lambda= 0.1$ for all \textit{i}.
Comparing Fig. \ref{hm_ensemble} to Fig. \ref{C3E_hm}, one can see that \textbf{OAC\textsuperscript{3}} does a better job, especially with the most difficult objects to be classified, 
showing that the information provided by the similarity matrix can improve the generalization capability of classifiers.

We also evaluate the performance of \textbf{OAC\textsuperscript{3}} for different proportions (from 1\% to 50\%) of 
training data. Fig. \ref{Comparison} summarizes the average accuracies (over 10 trials) achieved by \textbf{OAC\textsuperscript{3}}. 
The accuracies provided by the classifier ensemble, as well as by its \textit{best} individual component, are also shown for comparison purposes. 
The results obtained by \textbf{OAC\textsuperscript{3}} are consistently better than those achieved by the classifier ensemble. As expected, the curve 
for \textbf{OAC\textsuperscript{3}} shows that the less the amount of labeled objects, the greater are the benefits of using the information provided by 
the cluster ensemble. With 2\% of training data, the accuracies observed are 100\% in nine trials and 95\% in one trial. 
The mean and standard deviation are 99.5 and 1.59 respectively. This explains why the error bar exceeds 100\%.

\begin{figure}[h]
 \centering
 \includegraphics[bb=0 0 1012 626,scale=0.4]{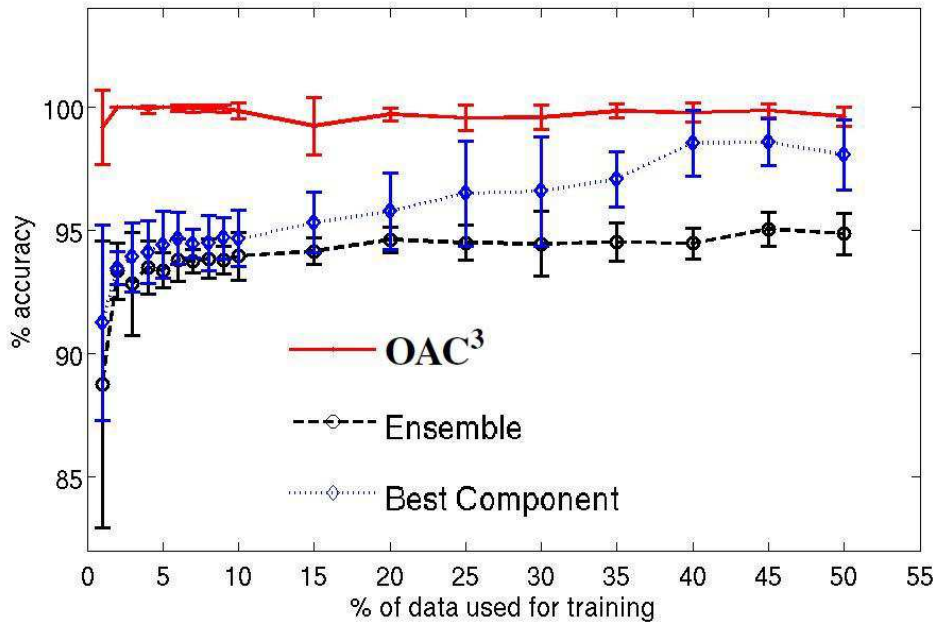}
 \caption{Average Accuracies and Standard Deviations.}
 \label{Comparison}
\end{figure}

\subsection{Sensitivity Analysis}
\label{Sensitivity_Analysis}

We perform a sensitivity analysis on the \textbf{OAC\textsuperscript{3}} parameters by using the same classification datasets employed in \cite{galf09}. 
These datasets represent eleven classification tasks from three real-world applications (\textit{20 Newsgroups}, \textit{Cora}, and \textit{DBLP}).
There are six datasets (News1 --- News6) for \textit{20 Newsgroups} and four datasets (Cora1 --- Cora4) for \textit{Cora}.
In each task, there is a target set on which the class labels should be predicted. 
In \cite{galf09}, two supervised models and two unsupervised models were used to obtain (on the target sets) class and cluster labels, respectively. 
These same class and cluster labels are used as inputs to \textbf{OAC\textsuperscript{3}}.
Then, we vary the \textbf{OAC\textsuperscript{3}} parameters and observe their respective accuracies.

In order to analyze the influence of the parameters $\alpha$ and $\lambda$ (recall that we set $\lambda_{i}^{(r)}=\lambda_{i}^{(l)}=\lambda$ for all \textit{i}), 
we consider that the algorithm converges when the relative difference of the objective function in two consecutive iterations is less than $\varepsilon=10^{-10}$.
By adopting this criterion, \textbf{OAC\textsuperscript{3}} usually converges after nine iterations (on average).
The algorithm has shown to be robust with respect to $\lambda$.
As far as $\alpha$ is concerned, for most of the datasets --- News1, News3, News4, News6, Cora1, Cora3, Cora4, and DBLP --- 
the classification accuracies achieved from \textbf{OAC\textsuperscript{3}} are better than those found by the classifier ensemble --- no matter 
the value chosen for $\alpha$. 
Figure \ref{fig:6} illustrates a typical accuracy surface for different values of $\lambda$ and $\alpha$.
It is worth mentioning that the accuracy surface tends to keep steady for $\alpha>1$ (\textit{i.e.}, the accuracies do not change significantly).
In particular, \textbf{OAC\textsuperscript{3}} was run for $\alpha=\{10;20;...;100;200;...;1000;100000\}$, for which the obtained results are the same as those achieved for $\alpha=1$ for any value of $\lambda$. This same observation holds for all the assessed datasets.
The interpretation for such results is that there is a threshold value for $\alpha$ that makes the second term of the objective function in (\ref{eqn:4}) dominating --- \textit{i.e.}, the information provided by the cluster ensemble is much more important than the information provided by the classifier ensemble.

We observed that for five datasets (News3, News6, Cora1, Cora3, and DBLP) any value of $\alpha>0.30$ provides the best classification accuracy. Thus, the algorithm can be robust with respect to the choice of its parameters for some datasets. For the datasets News2 and News5, some $\alpha$ values yield to accuracy deterioration, thereby suggesting that, depending on the value chosen for $\alpha$, the information provided by the cluster ensemble may hurt --- \textit{e.g.}, 
see Figure \ref{fig:7}. Finally, for Cora2, accuracy improvements were not observed, \textit{i.e.}, the accuracies provided by the classifier ensemble were always the best ones. 
This result suggests that the assumption that classes can be represented by means of clusters does not hold.


\begin{figure}
 \centering
 \includegraphics[scale=0.5]{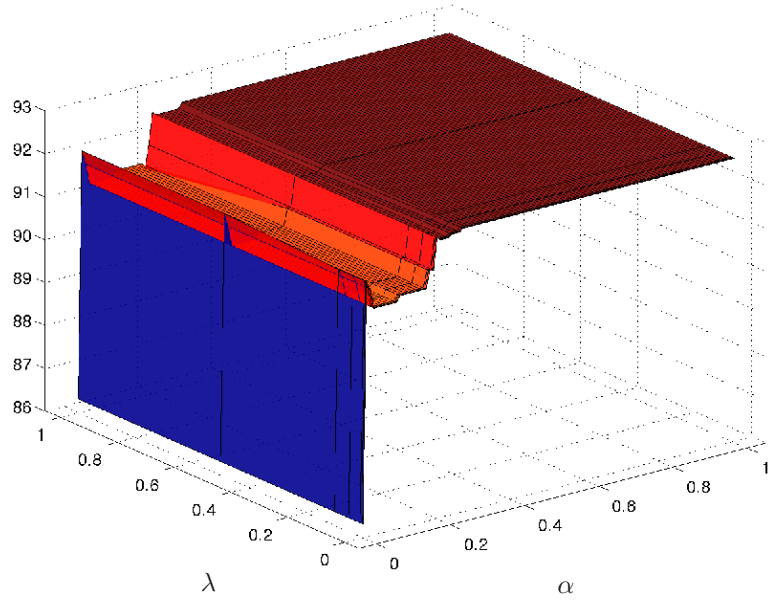}
   \begin{center} \vspace{-0.35cm} \hspace{-5.0cm}
$\lambda$ \end{center}
   \begin{center} \vspace{-0.55cm} \hspace{4.5cm}
$\alpha$ \end{center}
 \caption{Accuracy Surface -- News6}
 \label{fig:6}
\end{figure}

\begin{figure}
 \centering
 \includegraphics[scale=0.5]{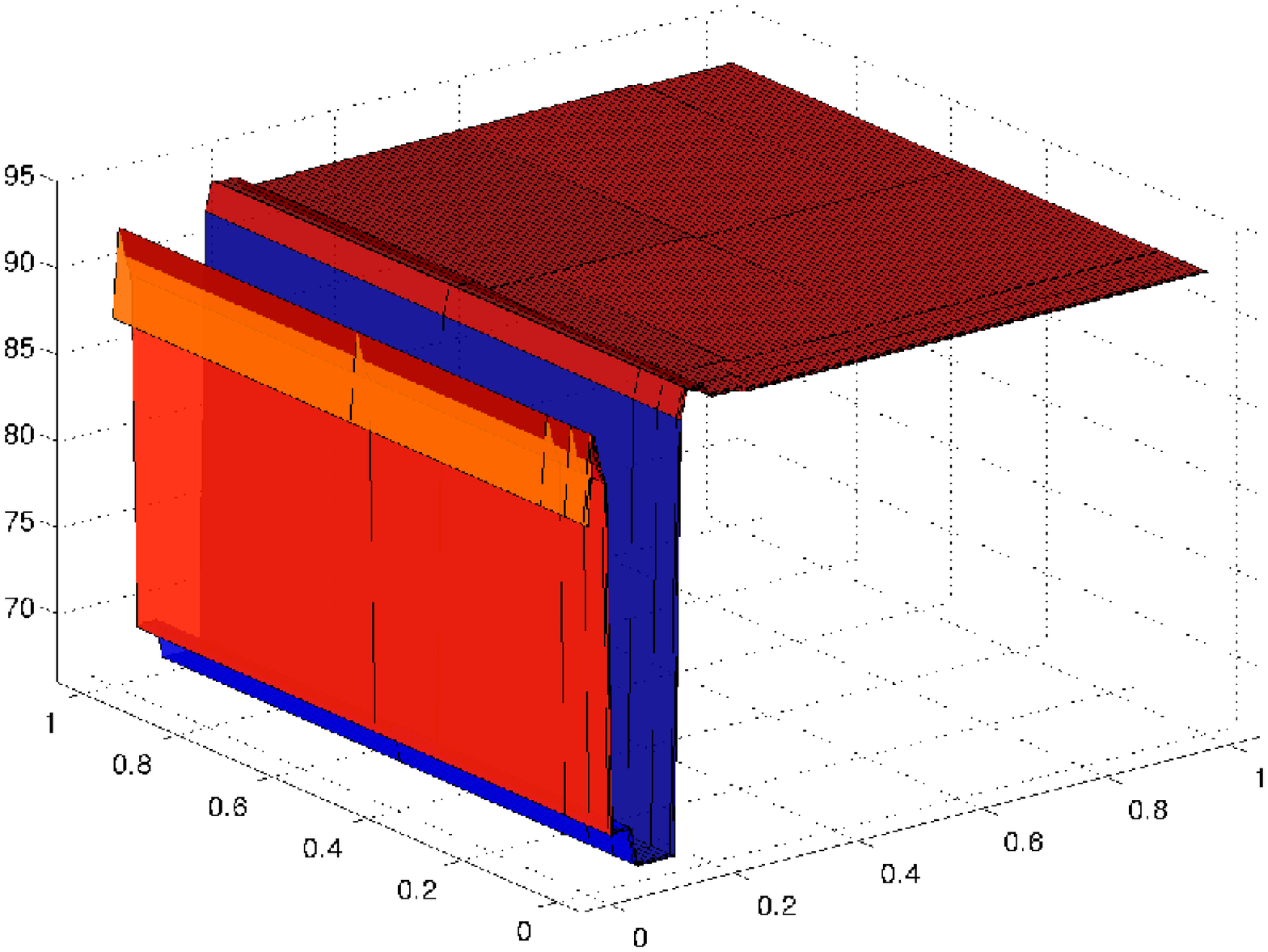}
   \begin{center} \vspace{-0.35cm} \hspace{-5.0cm}
$\lambda$ \end{center}
   \begin{center} \vspace{-0.55cm} \hspace{4.5cm}
$\alpha$ \end{center}
 \caption{Accuracy Surface -- News2}
 \label{fig:7}
\end{figure}

\begin{figure}
 \centering
 \includegraphics[scale=0.5]{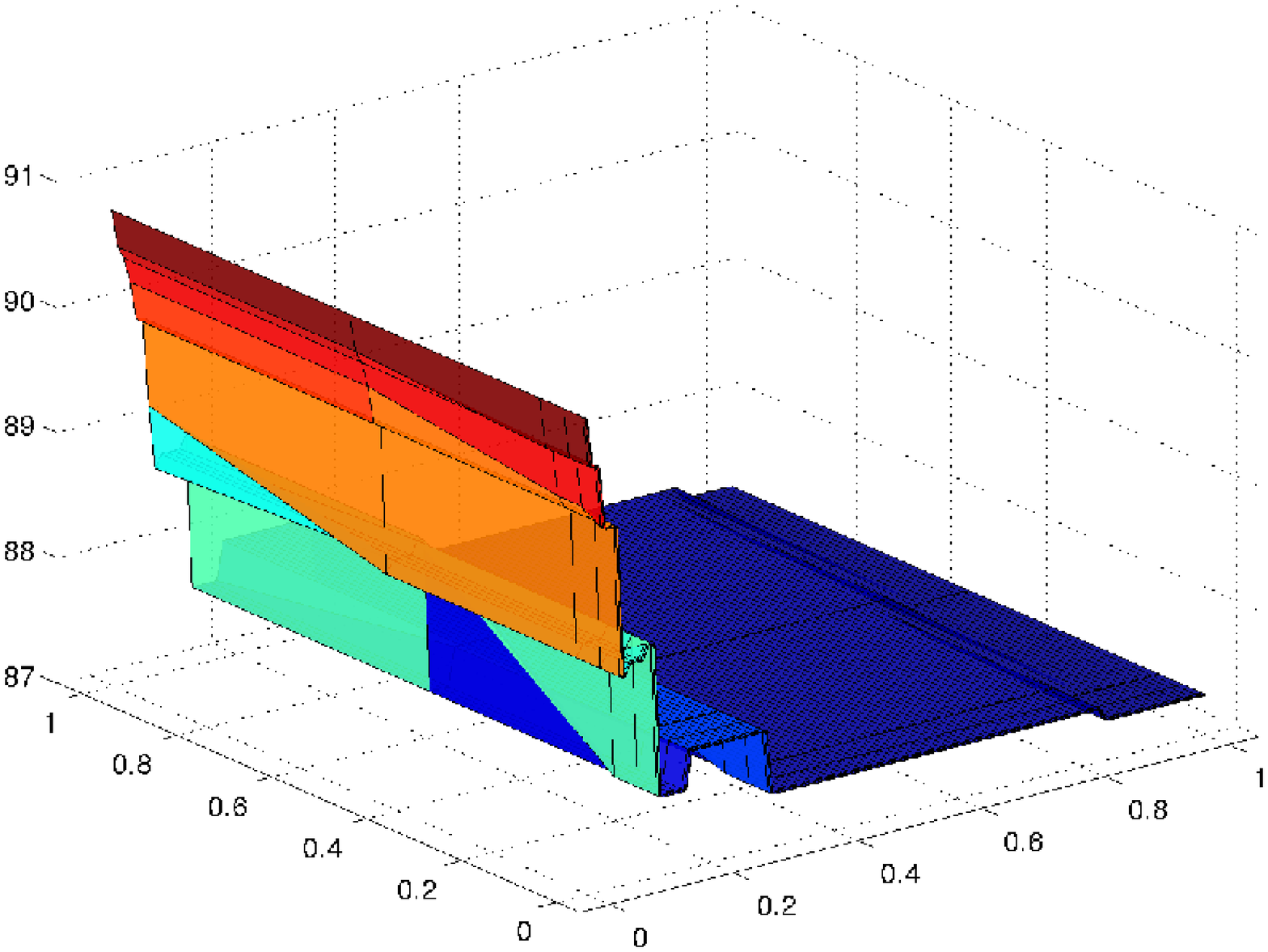}
   \begin{center} \vspace{-0.35cm} \hspace{-5.0cm}
$\lambda$ \end{center}
   \begin{center} \vspace{-0.55cm} \hspace{4.5cm}
$\alpha$ \end{center}
 \caption{Accuracy Surface -- Cora2}
 \label{fig:7}
\end{figure}

As expected, our experiments also show that the number of iterations may influence the performance of the algorithm. In particular, depending on the values chosen for $\alpha$, a high number of iterations may prejudice the obtained accuracies. Considering the best values obtained for $\alpha$ in our sensitivity analysis, we observed that, for all datasets, the best accuracies were achieved for less than 10 iterations.

By taking into account the results obtained in our sensitivity analyses, and recalling that fine tuning of the \textbf{OAC\textsuperscript{3}} parameters can be done by means of cross-validation, in the next section we compare the performance of \textbf{OAC\textsuperscript{3}} with the recently proposed \textbf{BGCM} \cite{galf09,Gao_TKDE}.


\subsection{Comparison with BGCM}
\label{BGCMComparison}

As discussed in Section \ref{sec:relatedwork}, \textbf{BGCM} is the algorithm most closely related to \textbf{OAC\textsuperscript{3}}. We evaluate \textbf{OAC\textsuperscript{3}} on the same classification datasets
employed to assess \textbf {BGCM} \cite{galf09,Gao_TKDE}. 
These datasets are those addressed in Section \ref{Sensitivity_Analysis}.
In \cite{galf09}, two supervised models (\textbf{M$_1$} and \textbf{M$_2$}) and two unsupervised models (\textbf{M$_3$} and \textbf{M$_4$}) were used to obtain (on the target sets) class and cluster labels, respectively. 
These same labels are used as inputs to \textbf{OAC\textsuperscript{3}}. In doing so, comparisons between \textbf{OAC\textsuperscript{3}} and \textbf {BGCM} are performed using exactly the same base models, which were trained in the same datasets\footnote{For these datasets, comparisons with \textbf{S\textsuperscript{3}VM} \cite{sind06} have not been performed because the raw data required for learning is not available.}. In other words, both \textbf{OAC\textsuperscript{3}} and \textbf {BGCM} receive the same inputs with respect to the components of the ensembles, from which consolidated classification solutions for the target sets are generated.

For the sake of compactness, the description of the datasets and learning models used in \cite{galf09} are not reproduced here, and the interested reader is referred to that paper for further details. However, the results for their four base models (\textbf{M$_1$},...,\textbf{M$_4$}), for \textbf{BGCM}, and for two well-known cluster ensemble approaches --- \textbf{MCLA} \cite{stgh02b} and \textbf{HBGF} \cite{febr04} --- are reproduced here for comparison purposes. Being cluster ensemble approaches, \textbf{MCLA} and \textbf{HBGF} ignore the class labels, considering that the four base models provide just cluster labels. Therefore, to evaluate classification accuracy obtained by these ensembles, the cluster labels are matched to the classes through an Hungarian method which favors the best possible class predictions. 
In order to run \textbf{OAC\textsuperscript{3}}, the supervised models (\textbf{M$_1$} and \textbf{M$_2$}) are fused to obtain class probability estimates 
for every instance in the target set. Also, the similarity matrix used by \textbf{OAC\textsuperscript{3}} is calculated by fusing the unsupervised 
models (\textbf{M$_3$} and \textbf{M$_4$}). 

The parameters of \textbf{OAC\textsuperscript{3}} have been chosen from the sensitivity analysis performed in Section \ref{Sensitivity_Analysis}.
However, for the experiments reported in this section we do not set particular values for each of the (eleven) studied datasets.
Instead, we have chosen a set of parameter values that result in good accuracies across related datasets.  
In particular the following pairs of ($\alpha,\lambda$)
are respectively used for the datasets \textit{News}, \textit{Cora}, and \textit{DBLP}:  ($4\times10^{-2}$,$10^{-2}$); ($10^{-4}$,$10^{-2}$); ($10^{-7}$, $10^{-3}$).
Such choices will hopefully show that one can get good results by using \textbf{OAC\textsuperscript{3}} without being (necessarily) picky about its parameter values --- thus these results are also complementary to the ones provided in Section \ref{Sensitivity_Analysis}.

\begin{table}
\tbl{Comparison of \textbf{OAC\textsuperscript{3}} with Other Algorithms --- Classification Accuracies (Best Results in Boldface).}
{
	\begin{tabular}{|c|c|c|c|c|c|c|c|c|c|c|c|}	
	\hline
        \hline 
	Method & News1 & News2 & News3 & News4 & News5 & News6 & Cora1 & Cora2 & Cora3 & Cora4 & DBLP\\
	\hline
	M$_1$ & 79.67 & 88.55 & 85.57 & 88.26 & 87.65 & 88.80 & 77.45 & 88.58 & 86.71 & 88.41 & 93.37 \\
	\hline
	M$_2$ &  77.21 & 86.11 & 81.34 & 86.76 & 83.58 & 85.63 & 77.97 & 85.94 & 85.08 & 88.79 & 87.66\\
	\hline
	M$_3$ &  80.56 & 87.96 & 86.58 & 89.83 & 87.16 & 90.20 & 77.79 & 88.33 & 86.46 & 88.13 & 93.82\\
	\hline
	M$_4$ & 77.70 & 85.71 & 81.49 & 84.67 & 85.43 & 85.78 & 74.76 & 85.94 & 78.10 & 90.16 & 79.49 \\
	\hline
	MCLA &  75.92 & 81.73 & 82.53 & 86.86 & 82.95 & 85.46 & 87.03 & 83.88 & 88.92 & 87.16 & 89.53\\
	\hline
	HBGF & 81.99 & 92.44 & 88.11 & 91.52 & 89.91 & 91.25 & 78.34 & 91.11 & 84.81 & 89.43 & 93.57 \\
	\hline
	BGCM & 81.28 & 91.01 & 86.08 & 91.25 & 88.64 & 90.88 & 86.87 & 91.55 & 89.65 & 90.90 & 94.17 \\
	\hline
	OAC\textsuperscript{3} & \textbf{85.01} & \textbf{93.64} & \textbf{89.64} & \textbf{93.80} & \textbf{91.22} & \textbf{92.59} & \textbf{88.54} & \textbf{90.79} & \textbf{90.60} & \textbf{91.49} & \textbf{94.38} \\
	\hline
        \hline
        \end{tabular}}
	\label{tab:accuracies}
\end{table}

The classification accuracies achieved by the studied methods are summarized in Table \ref{tab:accuracies}, where one can see that \textbf{OAC\textsuperscript{3}} 
shows the best accuracies for all datasets. 
In order to provide some reassurance about the validity and non-randomness of the obtained results, the outcomes 
of statistical tests, following the study in \cite{dems06}, are also reported. In brief, multiple algorithms are compared on multiple datasets by using 
the Friedman test, with a corresponding Nemenyi post-hoc test. 
The Friedman test is a non-parametric statistic test equivalent to the repeated-measures ANOVA. If the null hypothesis, which states that the algorithms under study have similar performances, is rejected, then the Nemenyi post-hoc test is used for pairwise comparisons between algorithms. 
The adopted statistical procedure indicates that the null hypothesis of equal accuracies --- considering the results obtained by the ensembles --- can 
be rejected at 10\% significance level. 
In pairwise comparisons, significant statistical differences are only 
observed between \textbf{OAC\textsuperscript{3}} and the other ensembles, \textit{i.e.},  there is no evidence that the accuracies of \textbf{MCLA}, 
\textbf{HBGF}, and \textbf {BGCM} are statistically different from one another.

\subsection{Comparison with \textbf{S\textsuperscript{3}VM}}
\label{SVMComparison}

We also compare \textbf{OAC\textsuperscript{3}} to a popular semi-supervised algorithm known as \textbf{S\textsuperscript{3}VM} \cite{sind06}. This algorithm is essentially a Transductive Linear Support Vector Machine (\textbf{SVM}) which can be viewed as a large scale implementation of the algorithm introduced in \cite{joac99}. For dealing with unlabeled data, it appends an additional term in the \textbf{SVM} objective function whose role is to drive the classification hyperplane towards low data density regions \cite{sind06}. The default parameter values have been used for \textbf{S\textsuperscript{3}VM}.

Six datasets are used in our experiments: \textit{Half-Moon} (see Section \ref{PedagogicalExample}), \textit{Circles} (which is a synthetic dataset that has two-dimensional instances that form 
two concentric circles --- one for each class), and four datasets from the \textit{Library for Support Vector Machines}\footnote{\url{http://www.csie.ntu.edu.tw/~cjlin/libsvm/}} --- \textit{Pima Indians Diabetes}, \textit{Heart}, \textit{German Numer}, and \textit{Wine}.
In order to simulate real-world classification problems where there is a very limited amount of labeled instances, small percentages (e.g., 2\%) of the instances are randomly selected for training, whereas the remaining instances are used for testing (target set). The amount of instances for training is chosen so that 
the pooled covariance matrix of the training set is positive definite. This \textit{restriction} comes from the use of an \textbf{LDA} classifier in the ensemble, and it imposes a lower bound on the number of training instances (7\% for \textit{Heart} 
and 10\% for \textit{German Numer}). We perform 10 trials for every proportion of instances in the training/target sets. The number of features are
2, 2, 8, 13, 24, 24 for Half-moon, Circles, Pima, Heart, German Numer and Wine respectively.


\begin{table*}[htbp]
\centering
\tbl{Comparison of \textbf{OAC\textsuperscript{3}} with \textbf{BGCM} and \textbf{S\textsuperscript{3}VM} --- Average Accuracies $\pm$(Standard Deviations).}
{
\begin{tabular}{|l|c|c|c|c|c|c|c|}
\hline
\hline
Dataset & \multicolumn{1}{l|}{$|\mathcal{X}|$} & Ensemble & Best Component& S\textsuperscript{3}VM & BGCM & OAC\textsuperscript{3}\\ \hline

Half-moon(2\%) & 784 & 92.53$(\pm1.83)$ & 93.02$(\pm0.82)$  & 99.61$(\pm0.09)$   & 92.16$(\pm1.47)$ & \textbf{99.64}$(\pm0.08)$\\ \hline
Circles(2\%) & 1568 & 60.03$(\pm8.44)$ & 95.74$(\pm5.15)$  & 54.35$(\pm4.47)$ & 78.67$(\pm0.54)$   & \textbf{99.61}$(\pm0.83)$\\ \hline
Pima(2\%)  & 745 & 68.16$(\pm5.05)$ &  69.93$(\pm3.68)$ & 61.67$(\pm3.01)$ & 69.21$(\pm4.83)$  &  \textbf{70.31}$(\pm4.44)$\\ \hline   
Heart(7\%)  &  251 & 77.77$(\pm2.55)$   & 79.22$(\pm2.20)$  & 77.07$(\pm4.77)$ & 82.78$(\pm4.82)$    & \textbf{82.85}$(\pm5.25)$\\ \hline
G. Numer(10\%)  &  900 & 70.96$(\pm1.00)$ & 70.19$(\pm1.52)$ &   73.00$(\pm1.50)$ &   73.70$(\pm1.06)$  &  \textbf{74.44}$(\pm3.44)$\\ \hline
Wine(10\%)  &  900 & 79.87$(\pm5.68)$ & 80.37$(\pm5.47)$ & 80.73$(\pm4.49)$ & 75.37$(\pm13.66)$  & \textbf{83.62}$(\pm6.27)$\\ \hline

\hline
\end{tabular}}
\label{table2}
\end{table*}

Considering \textbf{OAC\textsuperscript{3}}, the components of the classifier ensemble are chosen as previously described in Section \ref{PedagogicalExample}. 
Cluster ensembles are generated by means of multiple runs of $k$-means 
(10 data partitions for the two-dimensional datasets and 50 data partitions for \textit{Pima}, \textit{Heart}, \textit{German Numer}, and \textit{Wine}).

The parameters of \textbf{OAC\textsuperscript{3}} ($\alpha$ and $\lambda$) are optimized for better performance in each 
dataset using 5-fold cross-validation. The optimal values of $(\alpha, \lambda)$ for \textit{Half-moon}, \textit{Circles}, \textit{Pima}, \textit{Heart}, \textit{German Numer}, and \textit{Wine} are 
(0.05,0.1), (0.01,0.1), (0.002,0.1), (0.01,0.2), (0.01,0.1) and (0.01,0.1) respectively.
Table \ref{table2} shows that the accuracies obtained by \textbf{OAC\textsuperscript{3}} are good and consistently better than those achieved by both 
the classifier ensemble and its \textit{best} 
individual component. In addition, \textbf{OAC\textsuperscript{3}} shows better accuracies than both \textbf{S\textsuperscript{3}VM} and \textbf{BGCM} --- 
from the adopted statistical procedure \cite{dems06}, \textbf{OAC\textsuperscript{3}} exhibits significantly better accuracies at a significance level of $10\%$.

\subsection{Transfer Learning}
\label{TransferLearning}

Transfer learning emphasizes the transfer of knowledge across domains, tasks, and distributions that are similar but not the same \cite{sibe08}. We focus on learning scenarios where training and test distributions are different, as they represent (potentially) related but not 
identical tasks. It is assumed that the training and test domains involve the same class labels. The real-world datasets employed in our experiments are:

\textit{a) Text Documents} --- \cite{paya10}: From the well-known text collections \textit{20 newsgroup} and \textit{Reuters-21758}, nine cross-domain learning tasks are generated. The two-level hierarchy in both of these datasets is exploited to frame a learning task involving a top category classification problem with 
training and test data drawn from different sub categories --- \textit{e.g.}, to distinguish documents from two top newsgroup categories (rec and talk), the training set is built from ``rec.autos", ``rec.motorcycles", ``talk.politics", and ``talk.politics.misc", and the test set is formed from the sub-categories ``rec.sport.baseball", ``rec.sport.hockey", ``talk.politics.mideast", and ``talk.religions.misc". 
The \textit{Email spam} data set, released by ECML/PKDD 2006
discovery challenge, contains a training set of publicly available messages and three sets of email messages 
from individual users as test sets. 
The 4000 labeled examples in the training set and the 2500 test examples for each of the three
different users differ in the word distribution. A spam filter learned from public sources are used to 
test transfer capability on each of the users.


\begin{figure*}[htbp]
\centering
 \includegraphics[bb=0 0 1476 256, scale=0.25]{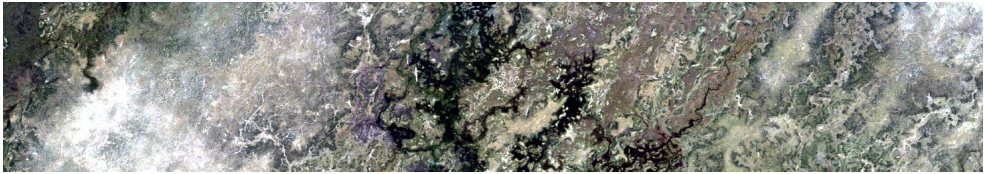}
 \caption{\textit{Botswana} May 2001.}
\label{bots_may}
\end{figure*}

\begin{figure*}[htbp]
\centering
 \includegraphics[bb=0 0 1476 256, scale=0.25]{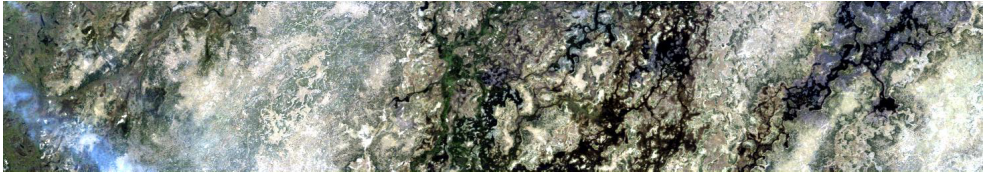}
 \caption{\textit{Botswana} June 2001.}
\label{bots_june}
\end{figure*}

\begin{figure*}[htbp]
\centering
 \includegraphics[bb=0 0 1476 256, scale=0.25]{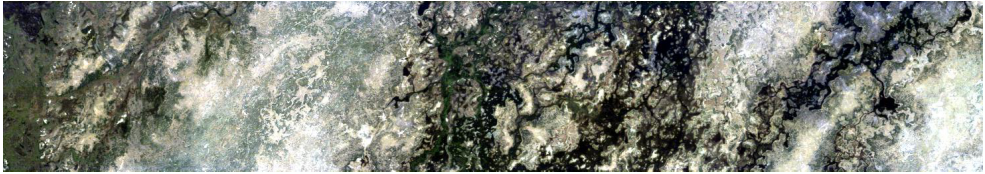}
 \caption{\textit{Botswana} July 2001.}
\label{bots_july}
\end{figure*}

\textit{b) Botswana} --- \cite{ragc06}: This is an application of transfer learning to the pixel-level classification of remotely sensed images, which provides a real-life scenario where such learning will be useful --- in contrast to the contrived setting of text classification, which is chosen as it has been used previously in \cite{daxy07}. It is relatively easy to acquire an image, but expensive to label each pixel manually, where images typically have about a million pixels and represent inaccessible terrain. 
Thus typically only  part of an image gets labeled. Moreover, when the satellite again flies over the same area, the new image can be quite different due to change of season, thus a classifier induced on the previous image becomes significantly degraded for the new task.
These hyperespectral data sets used are from a $1476 \times 256$ pixel study area located
in the Okavango Delta, Botswana. It has nine different land-cover types consisting of seasonal swamps, occasional swamps,
and drier woodlands located in the distal portion of the delta. 
Data from this region for different months (May, June and July) were obtained by the Hyperion sensor of the NASA EO-1 satellite for the calibration/validation portion of the mission in 2001. Data collected for each month was further segregated into two different areas.
While the May scene (Fig. \ref{bots_may}) is characterized by the onset of the annual flooding cycle and some newly burned areas,
the progression of the flood and the corresponding vegetation responses are seen in the June (Fig. \ref{bots_june}) and July (Fig. \ref{bots_july}) scenes.
The acquired raw data was further processed to produce 145 features. 
From each area of Botswana, different transfer learning tasks are generated: the classifiers are trained on either May, June or \{May $\cup$ June\} data and tested on either 
June or July data.  

For text data, we use logistic regression (\textbf{LR}), \textbf{SVM}, and Winnow (\textbf{WIN}) \cite{gafj08} as baseline classifiers. 
The CLUTO package (\url{http://www.cs.umn.edu/~karypis/cluto}) is used for clustering the target data into two clusters. 
We also compare \textbf{OAC\textsuperscript{3}} with two transfer learning algorithms from the literature --- Transductive Support Vector Machines (\textbf{TSVM}) \cite{joac99a} and 
the Locally Weighted Ensemble (\textbf{LWE}) \cite{gafj08}. We use Bayesian Logistic Regression \url{http://www.bayesianregression.org/}
for running the logistic regression classifier, LIBSVM (\url{http://www.csie.ntu.edu.tw/~cjlin/libsvm/}) for \textbf{SVM}, SNoW Learning Architecture 
\url{http://cogcomp.cs.illinois.edu/page/software_view/1} for \textbf{Winnow}, and SVM\textsuperscript{light} \url{http://svmlight.joachims.org/}
for transductive \textbf{SVM}. The posterior class probabilities from \textbf{SVM} are also obtained using the LIBSVM package with linear kernel. 
For SNoW, ``-S 3 -r 5" is used and the remaining parameters of all the packages are set to their default values.
The values of $(\alpha,\lambda)$, obtained by 10-fold cross-validation in source domain, are set as $(0.008,0.1)$ and $(0.11,0.1)$ for 
the transfer learning tasks corresponding to 20 \textit{Newsgroup} and \textit{Spam} datasets, respectively.
For \textit{Reuters}-21578, the best values of the parameters $(\alpha,\lambda)$ are found as $(0.009,0.1)$, $(0.0001,0.1)$, and $(0.08,0.1)$ for 
\textit{O vs Pe}, \textit{O vs Pl}, and \textit{Pe vs Pl}, respectively (see Table \ref{table:result2}).
For the hyperspectral data, we use two baseline classifiers: the well-known Na\"{\i}ve Bayes Wrapper (\textbf{NBW}) and the 
Maximum Likelihood (\textbf{ML}) classifier, which performs well when used with a best bases feature extractor \cite{kugc01}. 
The target set instances are clustered by $k$-means, varying \textit{k} from 50 to 70. \textbf{PCA} is also used for reducing 
the number of features employed by \textbf{ML}. In particular, for the hyperspectral 
data, cross-validation in the source domain does not result in very good performance. Therefore, 
we take 5\% labeled examples from each of the nine classes of the target data and tune the values of $\alpha$ and 
$\lambda$ based on the performance on these examples. The classifiers \textbf{NBW} or \textbf{ML}, however, are not 
retrained with these examples from the target domain and the accuracies reported in Table \ref{table:result2} are on the 
unlabeled examples only from the target domain. 

\begin{table*}[htp]
\centering
\tbl{Classification of \textit{20 Newsgroup}, \textit{Reuters-21758} and \textit{Spam} Data.}
{
\begin{tabular}{|l|c|c|c|c|c|c|c|c|}
\hline
\hline
\multicolumn{1}{|l|}{Dataset} & {Mode} & \multicolumn{1}{c|}{WIN}  & \multicolumn{1}{c|}{LR} & \multicolumn{1}{c|}{SVM} & \multicolumn{1}{c|}{Ensemble} 
& \multicolumn{1}{c|}{TSVM} & \multicolumn{1}{c|}{LWE} & \multicolumn{1}{c|}{OAC\textsuperscript{3}}\\ \hline
\multirow{6}{*}{20 Newsgroup}	
& C vs S &  66.61  & 67.17 &  67.02 &  69.58 & 76.97 & 77.07  & \textbf{91.25}\\  
& R vs T &  60.43  & 68.79 &  63.87 &  65.98 & 89.95 & 87.46  & \textbf{90.11}\\  
& R vs S &  80.11  & 76.51 &  71.40 &  77.39 & 89.96 & 87.81  & \textbf{92.90}\\ 
& S vs T &  73.93  & 72.16 &  71.51 &  75.11 & 85.59 & 81.99  & \textbf{91.83}\\ 
& C vs R &  89.00  & 77.36 &  81.50 &  85.18 & 89.64 & 91.09  & \textbf{93.75}\\ 
& C vs T &  93.41  & 91.76 &  93.89 &  93.48 & 88.26 & \textbf{98.90}  & 98.70\\\hline 
\multirow{3}{*}{Reuters-21758}	
& O vs Pe &  70.57 & 66.19 & 69.25 & 73.30 & 76.94 & 76.77  & \textbf{80.97} \\  
& O vs Pl &  65.10 & 67.87 & 69.88 & 69.21 & \textbf{70.08} & 67.59   & 68.91 \\  
& Pe vs Pl &  56.75 & 56.48 & 56.20 & 57.59  & 59.72 &  59.90  & \textbf{67.46}\\\hline  
\multirow{3}{*}{Spam}	
& spam 1 & 79.15  & 56.92  & 66.28 & 68.64  & 76.92  & 65.60	& \textbf{80.29}\\
& spam 2 & 81.15  & 59.76  & 73.15 & 75.07  & 84.92  & 73.36	& \textbf{87.05}\\
& spam 3 & 88.28  & 64.43  & 78.71 & 81.87  & 90.79  & \textbf{93.79}	& 91.27\\\hline\hline          
\end{tabular}}
\label{table:result1}
\end{table*}

The results for text data are reported in Table \ref{table:result1}. 
The different learning tasks corresponding to different pairs of 
categories are listed as ``Mode''. \textbf{OAC\textsuperscript{3}} improves the performance of the classifier ensemble 
(formed by combining  \textbf{WIN}, \textbf{LR} and \textbf{SVM} via output averaging)  for all learning tasks, except for \textit{O vs Pl}, where apparently the 
training and test distributions are similar. 
Also, the \textbf{OAC\textsuperscript{3}} accuracies are better than those achieved by both \textbf{TSVM} and \textbf{LWE} in 
most of the datasets. Except for \textbf{WIN}, 
the performances of the base classifiers and clustereres (and hence of \textbf{OAC\textsuperscript{3}}) are quite invariant, thereby resulting in very low standard deviations. 
The \textbf{OAC\textsuperscript{3}} accuracies are significantly better than those obtained by both \textbf{TSVM} and \textbf{LWE} (at $10\%$ significance level).

\begin{table*}[htp]
\centering
\tbl{Classification of Hyperspectral Data --- \emph{Botswana}.}
{
\begin{tabular}{|l|l|c|c|c|c|c|c|c|}
\hline
\hline
Dataset & Original to Target & \multicolumn{1}{|c|}{NBW} & \multicolumn{1}{c|}{NBW+OAC\textsuperscript{3}} & \multicolumn{1}{c|}{ML} 
& \multicolumn{1}{c|}{ML+OAC\textsuperscript{3}} & \multicolumn{1}{c|}{$\alpha$} & \multicolumn{1}{c|}{$\lambda$} & \multicolumn{1}{c|}{PCs}\\ \hline
\multirow{4}{*}{Area 1} & may to june & 70.68 & \textbf{72.61} $(\pm0.42)$ & 74.47 & \textbf{81.93} $(\pm0.52)$ & 0.0010 &0.1 & 9\\ 
 & may to july & 61.85 & \textbf{63.11} $(\pm0.29)$ & 58.58 & \textbf{64.32} $(\pm0.53$)& 0.0001 &0.2 & 12\\ 
 & june to july & 70.55 & \textbf{72.47} $(\pm0.17)$ & 79.71 & \textbf{80.06} $(\pm0.26)$& 0.0012 &0.1 & 127\\ 
 & may+june to july & 75.53 & \textbf{80.53} $(\pm0.31)$& 85.78 & \textbf{85.91} $(\pm0.23)$& 0.0008 &0.1 & 123\\ \hline
\multirow{4}{*}{Area 2} & may to june & 66.10 & \textbf{71.02} $(\pm0.28)$ & 70.22 & \textbf{81.48} $(\pm0.43)$& 0.0070 &0.1 & 9\\ 
 & may to july & 61.55 & \textbf{63.74} $(\pm0.14)$& 52.78 & \textbf{64.15} $(\pm0.22)$& 0.0001 &0.2 & 12\\ 
 & june to july & 54.89 & \textbf{57.65} $(\pm0.53)$& 75.62 & \textbf{77.04} $(\pm0.37)$& 0.0060 &0.1 & 80\\ 
 & may+june to july  & 63.79 & \textbf{64.58}  $(\pm0.16)$& 77.33 & \textbf{79.59} $(\pm0.23)$& 0.0040 &0.1 & 122\\ \hline\hline
\end{tabular}}
\label{table:result2}
\end{table*}

Table \ref{table:result2} 
reports the results for the hyperspectral data. 
The parameter values ($\alpha,\lambda$) for best performance of \textbf{OAC\textsuperscript{3}} are also presented alongside. 
Note that \textbf{OAC\textsuperscript{3}} provides consistent accuracy improvements for both \textbf{NBW} and \textbf{ML}\footnote{Standard deviations of the accuracies from \textbf{NBW} and \textbf{ML} are close to $0$ and hence not shown.}. In pairwise comparisons, the accuracies provided by 
\textbf{OAC\textsuperscript{3}} are significantly better than those obtained by both \textbf{NBW} and \textbf{ML} (at $10\%$ significance level). The column ``PCs'' indicates the number of principal 
components used to project the data.
\section{Concluding Remarks}
\label{sec:conc}

We presented a general framework for combining classifiers and clusterers to address  
semi-supervised and transfer learning problems. 
The optimization algorithm assumes closed form updates, facilitates parallelization of the same and, therefore, 
is extremely convenient in handling large scale data -- specially with a linear rate of convergence. The 
proofs for the convergence are quite novel and generalize across a wide variety of Bregman divergences, facilitating 
one to use proper divergence measure based on the application domain and subsuming many other existing graph based semi-supervised learning algorithms as special cases. 
The proposed framework has been empirically shown to outperform a variety 
of algorithms [\citeNP{Gao_TKDE}; \citeNP{sind06}; \citeNP{gafj08}] in both semi-supervised and transfer learning problems.

There are few aspects that can be further explored. For example, the impact of the number of classifiers and clusterers in \textbf{OAC\textsuperscript{3}} deserves further investigation. 
In addition, a more extensive study across a wide variety of problem domains will reveal the capabilities as well as potential limitations of the framework.

\appendix
\section*{APPENDIX}
\section{Proofs for Convergence of \textbf{OAC\textsuperscript{3}}}
\label{app:convergence}

\begin{lemma}
\label{Jlemma}
The objective function $J$ used in Eq. (\ref{eqn:mainobj}) is separately and jointly strictly convex over $\mathcal{S}^{n}\times\mathcal{S}^{n}$.  Also,
$J$ is jointly lower-semi-continuous w.r.t $\mathbf{y}^{(l)}$ and $\mathbf{y}^{(r)}$.   
\end{lemma}
\begin{proof}
\begin{enumerate}
 \item From the property (a) in Section \ref{sec:convergence}, one can see 
that $J$ is strictly convex w.r.t $\mathbf{y}^{(l)}$ and $\mathbf{y}^{(r)}$ separately. 
From the same property the first 
term $f_{1}(\mathbf{y}^{(r)}) = \displaystyle\sum_{i=1}^{n}d_{\phi}(\boldsymbol{\pi}_{i}, \mathbf{y}_{i}^{(r)})$ in $J$ is strictly convex w.r.t. $\mathbf{y}^{(r)}$.  
The $2^{\text{nd}}$ and $3^{\text{rd}}$ terms in the objective function can collectively be represented by 
$f_{2}(\mathbf{y}^{(l)}, \mathbf{y}^{(r)})$. This function is jointly convex by property (b) but is not 
necessarily jointly strictly convex.
Suppose $(\mathbf{y}^{1, (l)}, \mathbf{y}^{1, (r)}), (\mathbf{y}^{2, (l)}, \mathbf{y}^{2, (r)})\in\mathcal{S}^{n}\times\mathcal{S}^{n}$ and 
$0<w<1$. Then, we have:
\begin{eqnarray*}
f_{1}(w\mathbf{y}^{1, (r)}+(1-w)\mathbf{y}^{2, (r)}) &<& wf_{1}(\mathbf{y}^{1, (r)}) + (1-w)f_{1}(\mathbf{y}^{2, (r)})\\ \nonumber
f_{2}(w(\mathbf{y}^{1, (l)}, \mathbf{y}^{1, (r)}) + (1-w)(\mathbf{y}^{2, (l)}, \mathbf{y}^{2, (r)})) &\le& wf_{2}(\mathbf{y}^{1, (l)}, \mathbf{y}^{1, (r)}) 
+ (1-w)f_{2}(\mathbf{y}^{2, (l)}, \mathbf{y}^{2, (r)}) \nonumber .
\end{eqnarray*}
Now, it follows that:
\begin{eqnarray}
&& J(w(\mathbf{y}^{1, (l)}, \mathbf{y}^{1, (r)})+(1-w)(\mathbf{y}^{2, (l)}, \mathbf{y}^{2, (r)}))\\ \nonumber
&=&  f_{1}(w\mathbf{y}^{1, (r)}+(1-w)\mathbf{y}^{2, (r)}) + f_{2}(w(\mathbf{y}^{1, (l)}, \mathbf{y}^{1, (r)})+(1-w)(\mathbf{y}^{2, (l)}, \mathbf{y}^{2, (r)})) \\ \nonumber
&<& wf_{1}(\mathbf{y}^{1, (r)}) + (1-w)f_{1}(\mathbf{y}^{2, (r)}) + wf_{2}(\mathbf{y}^{1, (l)}, \mathbf{y}^{1, (r)}) + 
(1-w)f_{2}(\mathbf{y}^{2, (l)}, \mathbf{y}^{2, (r)}) \\ \nonumber
&=& wJ(\mathbf{y}^{1, (l)}, \mathbf{y}^{1, (r)})+
(1-w)J(\mathbf{y}^{2, (l)}, \mathbf{y}^{2, (r)}), 
\end{eqnarray}
which implies that $J$ is jointly strictly convex.
 \item To prove that $J(\mathbf{y}^{(l)}, \mathbf{y}^{(r)})$ is lower-semi-continuous in $\mathbf{y}^{(l)}$ 
and $\mathbf{y}^{(r)}$ jointly, we observe that
\begin{eqnarray}
&& \underset{(\mathbf{y}^{(l)}, \mathbf{y}^{(r)}) \to (\mathbf{y}^{0, (l)}, \mathbf{y}^{0, (r)})}
{\operatorname{lim \text{ }inf }} J(\mathbf{y}^{0, (l)}, \mathbf{y}^{0, (r)}) \\ \nonumber
&=& 
\left[\displaystyle\sum_{i=1}^{n} \underset{\mathbf{y}_{i}^{(r)} \to \mathbf{y}_{i}^{0,(r)}}{\operatorname{lim \text{ }inf }} 
d_{\phi}(\boldsymbol{\pi}_{i}, \mathbf{y}_{i}^{(r)})
+ \alpha\displaystyle\sum_{i,j=1}^{n} s_{ij} \underset{(\mathbf{y}_{i}^{(l)}, \mathbf{y}_{j}^{(r)}) \to (\mathbf{y}_{i}^{0,(l)}, \mathbf{y}_{j}^{0,(r)})}
{\operatorname{lim \text{ }inf }} d_{\phi}(\mathbf{y}_{i}^{(l)}, \mathbf{y}_{j}^{(r)}) \right. \\ \nonumber
&+& \left. \lambda\displaystyle\sum_{i=1}^{n}\underset{(\mathbf{y}_{i}^{(l)}, \mathbf{y}_{i}^{(r)}) \to (\mathbf{y}_{i}^{0,(l)}, \mathbf{y}_{i}^{0,(r)})}
{\operatorname{lim \text{ }inf }} d_{\phi}(\mathbf{y}_{i}^{(l)}, \mathbf{y}_{i}^{(r)})\right]\\ \nonumber
&\ge& \left[\displaystyle\sum_{i=1}^{n}d_{\phi}(\boldsymbol{\pi}_{i}, \mathbf{y}_{i}^{0, (r)})
+ \alpha\displaystyle\sum_{i,j=1}^{n} s_{ij}d_{\phi}(\mathbf{y}_{i}^{0, (l)}, \mathbf{y}_{j}^{0, (r)})
+ \lambda\displaystyle\sum_{i=1}^{n}d_{\phi}(\mathbf{y}_{i}^{0, (l)}, \mathbf{y}_{i}^{0, (r)})\right]\\ \nonumber
&=&  J(\mathbf{y}^{0, (l)}, \mathbf{y}^{0, (r)}).
\end{eqnarray}
The inequality in the $3^{\text{rd}}$ step follows from the lower semi continuity of $d_{\phi}(.,.)$ in Section \ref{sec:convergence} 
(Property (d)).
\end{enumerate}
\end{proof}

The following theorem helps prove that the objective function $J$ can be seen as part of a Bregman divergence.
\begin{theorem}[\cite{bamd05}]
\label{iffBregman}
A divergence $d:\mathcal{S}\times\text{ri}(\mathcal{S})\to[0, \infty)$ is a Bregman divergence if and only if 
$\exists\mathbf{a} \in \text{ri}(\mathcal{S})$ such that the function $\phi_{\mathbf{a}}(\mathbf{p}) = d(\mathbf{p}, \mathbf{a})$ satisfies the following 
conditions:
\begin{enumerate}
 \item $\phi_{\mathbf{a}}$ is strictly convex on $\mathcal{S}$. 
 \item $\phi_{\mathbf{a}}$ is differentiable on $\text{ri}(\mathcal{S})$.
 \item $d(\mathbf{p},\mathbf{q}) = d_{\phi_{\mathbf{a}}}(\mathbf{p},\mathbf{q}), \forall \mathbf{p} \in \mathcal{S}, \mathbf{q} \in 
\text{ri}(\mathcal{S})$ where $d_{\phi_{\mathbf{a}}}$ is the Bregman divergence associated with $\phi_{\mathbf{a}}$.
\end{enumerate}
\end{theorem}

We now introduce a function $\tilde{J}:\mathcal{S}^{n}\times \mathcal{S}^{n}\to [0,\infty)$ that is defined as follows:
\begin{equation}
\label{eqn:maintildeobj}
 \tilde{J}(\mathbf{y}^{(r)\prime}, \mathbf{y}^{(r)}) 
= \left[\displaystyle\sum_{i=1}^{n}d_{\phi}(\mathbf{y}_{i}^{(r)\prime}, \mathbf{y}_{i}^{(r)})
+ \alpha\displaystyle\sum_{i,j=1}^{n} s_{ij}d_{\phi}(\mathbf{y}_{i}^{(r)\prime}, \mathbf{y}_{j}^{(r)})
+ \lambda\displaystyle\sum_{i=1}^{n}d_{\phi}(\mathbf{y}_{i}^{(r)\prime}, \mathbf{y}_{i}^{(r)})\right].
\end{equation}
Note that $\tilde{J}$ is different from $J$ defined in Eq. (\ref{eqn:mainobj}). The left arguments in the divergences of the 
first term of $J$ are $\boldsymbol{\pi}_{i}$'s which are assumed to be fixed.
\begin{lemma}
\label{Jtildeprop}
 $\tilde{J}$ satisfies properties (a) and (b) in Section \ref{sec:convergence}.
\end{lemma}
\begin{proof}
 The proof is direct from the definition of $\tilde{J}$.
\end{proof}

Further assume: 
\begin{eqnarray*}
\mathbf{p} &=& \big((\boldsymbol{\pi}_{i})_{i=1}^{n}, (\mathbf{y}_{i}^{(l)})_{i=1}^{n}, ((\mathbf{y}_{i}^{(l)})_{j=1}^{n-1})_{i=1}^{n}\big),\\ \nonumber 
\mathbf{q} &=& \big((\mathbf{y}_{i}^{(r)})_{i=1}^{n}, (\mathbf{y}_{i}^{(r)})_{i=1}^{n}, ((\mathbf{y}_{j}^{(r)})_{j=1, j\neq i}^{n-1})_{i=1}^{n}\big),\\ \nonumber 
\mathbf{q}^{\prime} &=& \big((\mathbf{y}_{i}^{(r)\prime})_{i=1}^{n}, (\mathbf{y}_{i}^{(r)\prime})_{i=1}^{n}, ((\mathbf{y}_{j}^{(r)\prime})_{j=1, j\neq i}^{n-1})_{i=1}^{n}\big),
\end{eqnarray*}
with $\mathbf{y}_{i}^{(l)}, \mathbf{y}_{i}^{(r)}, \mathbf{y}_{i}^{(r)\prime}\in\mathcal{S}\text{ }\forall i$.
The vectors $\mathbf{p}$, $\mathbf{q}$ and $\mathbf{q}^{\prime}$ are each of dimension $kn(n+1)$ and formed by concatenating 
vectors from the set $\mathcal{S}$. $(\mathbf{y})_{i=1}^{n}$ implies that a new vector is created by repeating $\mathbf{y}$ 
for $n$ times.
For ease of understanding, we also define $\mathcal{A} = \{\mathbf{p}:\mathbf{y}^{(l)} \in\mathcal{S}^{n}\}$, $\mathcal{B} = \{\mathbf{q}:\mathbf{y}^{(r)} \in\mathcal{S}^{n}\}$.
We will assume that whenever a point $\mathbf{y}^{(l)}\in\mathcal{S}^{n}$ is 
mapped to a point $\mathbf{p}\in\mathcal{A}$, $\mathbf{p} = \mathbb{A}(\mathbf{y}^{(l)})$. Similarly, 
$\mathbf{q} = \mathbb{B}(\mathbf{y}^{(r)})$ whenever $\mathbf{y}^{(r)}\in\mathcal{S}^{n}$ is 
mapped to $\mathbf{q}\in\mathcal{B}$. Indeed, both $\mathbb{A}$ and $\mathbb{B}$ are bijective mappings. 

\begin{example}
To explain the mappings $\mathbb{A}$ and $\mathbb{B}$ more clearly, we consider the following example. 
Let $n=3$ and $\mathbf{y}^{(l)}, \mathbf{y}^{(r)}, \mathbf{y}^{(r)\prime} \in\mathcal{S}^{3}$. Here, 
$\mathbf{y}^{(l)} = \big(\mathbf{y}_{1}^{(l)},\mathbf{y}_{2}^{(l)},\mathbf{y}_{3}^{(l)}\big)$ -- a concatenation of 
three vectors $\mathbf{y}_{1}^{(l)}$, $\mathbf{y}_{2}^{(l)}$ and $\mathbf{y}_{3}^{(l)}$ (corrpesponding to three 
instances) each of which belongs to $\mathcal{S}\subseteq \mathbb{R}^{k}$. Similarly,
$\mathbf{y}^{(r)} = \big(\mathbf{y}_{1}^{(r)},\mathbf{y}_{2}^{(r)},\mathbf{y}_{3}^{(r)}\big)$ and 
$\mathbf{y}^{(r)\prime} = \big(\mathbf{y}_{1}^{(r)\prime},\mathbf{y}_{2}^{(r)\prime},\mathbf{y}_{3}^{(r)\prime}\big)$.
The vector $\mathbf{p}$, formed by the transformation $\mathbb{A}$ on $\mathbf{y}^{(l)}$, takes the following form:
\begin{equation*}
 \mathbf{p} = \big(\boldsymbol{\pi}_{1}, \boldsymbol{\pi}_{2}, \boldsymbol{\pi}_{3}, \mathbf{y}_{1}^{(l)}, \mathbf{y}_{2}^{(l)}, \mathbf{y}_{3}^{(l)},
\mathbf{y}_{1}^{(l)}, \mathbf{y}_{1}^{(l)}, \mathbf{y}_{2}^{(l)}, \mathbf{y}_{2}^{(l)}, \mathbf{y}_{3}^{(l)}, \mathbf{y}_{3}^{(l)}\big)
\end{equation*} 
Note that this vector has $12$ elements each of dimension $k$ and hence the dimension of the whole vector is of the form
$kn(n+1)$. Similarly,
 \begin{equation*}
 \mathbf{q} = \mathbb{B}(\mathbf{y}^{(r)}) = \big(\mathbf{y}_{1}^{(r)}, \mathbf{y}_{2}^{(r)}, \mathbf{y}_{3}^{(r)}, \mathbf{y}_{1}^{(r)}, \mathbf{y}_{2}^{(r)}, \mathbf{y}_{3}^{(r)},
\mathbf{y}_{2}^{(r)}, \mathbf{y}_{3}^{(r)}, \mathbf{y}_{1}^{(r)}, \mathbf{y}_{3}^{(r)}, \mathbf{y}_{1}^{(r)}, \mathbf{y}_{2}^{(r)}\big),
\end{equation*}
and,
\begin{equation*}
 \mathbf{q}^{\prime} = \mathbb{B}(\mathbf{y}^{(r)\prime}) = \big(\mathbf{y}_{1}^{(r)\prime}, \mathbf{y}_{2}^{(r)\prime}, \mathbf{y}_{3}^{(r)\prime}, 
\mathbf{y}_{1}^{(r)\prime}, \mathbf{y}_{2}^{(r)\prime}, \mathbf{y}_{3}^{(r)\prime},
\mathbf{y}_{2}^{(r)\prime}, \mathbf{y}_{3}^{(r)\prime}, \mathbf{y}_{1}^{(r)\prime}, \mathbf{y}_{3}^{(r)\prime}, \mathbf{y}_{1}^{(r)\prime}, 
\mathbf{y}_{2}^{(r)\prime}\big).
\end{equation*} 
\qed
\end{example}

Now, in light of Theorem \ref{iffBregman}, the following corollary is introduced.
\begin{corollary}
\label{app_additiveBregman} 
If a mapping $d: (\mathcal{A}\cup\mathcal{B})\times\mathcal{B}\to [0, \infty)$ is defined as:
\begin{equation}
\label{eqn:ddef}
d(\mathbf{r},\mathbf{q}) = 
\left\{
\begin{array}{ll}
 d(\mathbf{p},\mathbf{q}) = J(\mathbf{y}^{(l)}, \mathbf{y}^{(r)})& \mbox{if } \mathbf{r} = \mathbf{p}\in\mathcal{A}\\
 d(\mathbf{q}^{\prime},\mathbf{q}) = \tilde{J}(\mathbf{y}^{(r)\prime}, \mathbf{y}^{(r)})& \mbox{if } \mathbf{r} = \mathbf{q}^{\prime}\in\mathcal{B}
\end{array}
\right.  
\end{equation}
then $d$ is a Bregman divergence.
\end{corollary}
\begin{proof}
 We show that conditions (a), (b) and (c) of Theorem \ref{iffBregman} are satisfied for $d$.
 \begin{enumerate}
  \item  Since $d_{\phi}$ is a Bregman divergence, $\exists\mathbf{a}\in\text{ri}(\mathcal{S})$ such that conditions (a), (b) and (c) are 
satisfied in corollary \ref{iffBregman} pertaining to this divergence. Note that
$\mathbf{p}\in\mathcal{A}$ and $\mathbf{q}^{\prime}, \mathbf{q}\in\mathcal{B}$.
Assume  $\mathbf{a}^{\prime}=\mathbb{B}((\mathbf{a})_{i=1}^{n})\in\mathcal{B}\subset (\mathcal{A}\cup\mathcal{B})$. 
We now define
\begin{equation}
\label{eqn:psidef}
\psi_{\mathbf{a}^{\prime}}(\mathbf{r}) =
\left\{
\begin{array}{ll}
 \psi_{\mathbf{a}^{\prime}}(\mathbf{p}) = d(\mathbf{p},\mathbf{a}^{\prime}) = J(\mathbf{y}^{(l)}, \mathbf{y}^{(r)}) & \mbox{if } \mathbf{r} = \mathbf{p}\in\mathcal{A}\\
 \psi_{\mathbf{a}^{\prime}}(\mathbf{q}^{\prime}) = d(\mathbf{q}^{\prime},\mathbf{a}^{\prime}) = \tilde{J}(\mathbf{y}^{(r)\prime}, \mathbf{y}^{(r)})& 
\mbox{if } \mathbf{r} = \mathbf{q}^{\prime}\in\mathcal{B} 
\end{array}
\right.
\end{equation}
Since each of $d_{\phi}(.,\mathbf{a})$ is strictly convex over $\mathcal{S}^{n}$ in Eq. (\ref{eqn:mainobj}) and Eq. (\ref{eqn:maintildeobj}), 
$\psi_{\mathbf{a}^{\prime}}$ is also strictly convex on $\mathcal{A}\cup\mathcal{B}$. Note the emphasis on $\mathcal{B}\subset (\mathcal{A}\cup\mathcal{B})$ in the definition 
of $\mathbf{a}^{\prime}$ which just ensures that all conditions in Theorem \ref{iffBregman} are satisfied.
 \item Again, this is a direct consequence from Eq. (\ref{eqn:mainobj}) and Eq. (\ref{eqn:maintildeobj}). 
Since, by the strict convexity of $\phi(.)$, each of $d_{\phi}(.,\mathbf{a})$ is differentiable 
over $\text{ri}(\mathcal{S}^{n})$, $\psi_{\mathbf{a}^{\prime}}$ is also differentiable over $\text{ri}(\mathcal{A}\cup \mathcal{B})$. 
Note that we have a bijective mapping of elements from $\mathcal{S}^{n}$ to $\mathcal{A}\cup\mathcal{B}$, 
and hence $\text{ri}(\mathcal{S}^{n})$ gets mapped to $\text{ri}(\mathcal{A}\cup \mathcal{B})$.
\item We have $\forall \text{ }(\mathbf{p}, \mathbf{q})\in \mathcal{A}\times \mathcal{B}$,
\begin{eqnarray*}
&& d_{\psi_{\mathbf{a}^{\prime}}}(\mathbf{p}, \mathbf{q}) = \left[\psi_{\mathbf{a}^{\prime}}(\mathbf{p}) - \psi_{\mathbf{a}^{\prime}}(\mathbf{q}) 
- \langle \grad_{\psi_{\mathbf{a}^{\prime}}}(\mathbf{q}), (\mathbf{p}-\mathbf{q})\rangle \right]\\ \nonumber
&=& \displaystyle\sum_{i=1}^{n}\left[d_{\phi}(\boldsymbol{\pi}_{i}, \mathbf{a})-d_{\phi}(\mathbf{y}_{i}^{(r)}, \mathbf{a})\right] 
+\alpha \displaystyle\sum_{i,j=1}^{n}s_{ij}\left[d_{\phi}(\mathbf{y}_{i}^{(l)}, \mathbf{a})-d_{\phi}(\mathbf{y}_{j}^{(r)}, \mathbf{a})\right]\\ \nonumber
&+&\lambda\displaystyle\sum_{i=1}^{n}\left[d_{\phi}(\mathbf{y}_{i}^{(l)}, \mathbf{a})-d_{\phi}(\mathbf{y}_{i}^{(r)}, \mathbf{a})\right]
- \langle \grad_{\psi_{\mathbf{a}^{\prime}}}(\mathbf{q}), (\mathbf{p}-\mathbf{q})\rangle \\ \nonumber
&=& \displaystyle\sum_{i=1}^{n}\left[d_{\phi}(\boldsymbol{\pi}_{i}, \mathbf{y}_{i}^{(r)}) 
+\alpha \displaystyle\sum_{i,j=1}^{n}s_{ij}d_{\phi}(\mathbf{y}_{i}^{(l)}, \mathbf{y}_{j}^{(r)})
+\lambda\displaystyle\sum_{i=1}^{n}d_{\phi}(\mathbf{y}_{i}^{(l)}, \mathbf{y}_{i}^{(r)})\right]\\ \nonumber 
&=& d(\mathbf{p}, \mathbf{q}).
\end{eqnarray*}
The second step follows from the definition of $\psi(.)$ in Eq. (\ref{eqn:psidef})
and the last step follows from the definition of $d(\mathbf{p},\mathbf{q})$ in Eq. (\ref{eqn:ddef}). 
The equality $d_{\psi_{\mathbf{a}^{\prime}}}(\mathbf{q}^{\prime}, \mathbf{q}) = d(\mathbf{q}^{\prime}, \mathbf{q})$ 
$\forall (\mathbf{q}^{\prime},\mathbf{q})\in\mathcal{B}\times\mathcal{B}$ can similarly be 
proved. Therefore, combining the two results, we have 
$d_{\psi_{\mathbf{a}^{\prime}}}(\mathbf{r}, \mathbf{q}) = d(\mathbf{r}, \mathbf{q})$ $\forall 
(\mathbf{r},\mathbf{q})\in(\mathcal{A}\cup\mathcal{B})\times\mathcal{B}$. With a slight 
abuse of notation, henceforth, we will denote the mapping $\psi_{\mathbf{a}^{\prime}}$ by $\psi$
with an implicit assumption of the existence of an $\mathbf{a}^{\prime}\in\mathcal{B}$ as described before.
\end{enumerate}
We will see next that we require some definition of $\psi_{\mathbf{a}^{\prime}}(\mathbf{q})$ for 
$\mathbf{q}\in\mathcal{B}$ and this explains the definition of $d(\mathbf{r},\mathbf{q})$ in Eq. (\ref{eqn:ddef}) 
for the case when $\mathbf{r}=\mathbf{q}^{\prime}\in\mathcal{B}$.
\end{proof}

\begin{lemma}
\label{lemmadpsi}
 $d_{\psi}$ satisfies properties (a) and (b) in Section \ref{sec:convergence}. 
\end{lemma}
\begin{proof}
\begin{enumerate}
 \item One can see that $d_{\psi}$ is strictly convex separately w.r.t its arguments from its definition in Eq. (\ref{eqn:ddef}). Since 
each of $J$ and $\tilde{J}$ 
is strictly convex separately w.r.t the arguments and $\mathbb{A}$ and $\mathbb{B}$ are bijective mappings, $d_{\psi}$ is strictly convex separately
w.r.t. $\mathbf{r}$ and $\mathbf{q}$. 
 \item The joint convexity of $d_{\psi}$ also follows directly from its definition and the joint convexity of $J$ and $\tilde{J}$.
\end{enumerate}
\end{proof}

At this point, we reiterate that defining $d_{\psi}$ as in Eq. (\ref{eqn:ddef}) helps in proving some interesting properties of 
$J$ in a very elegant way. We, in fact, treat 
$d_{\psi}$ as a surrogate for $J$, establish two specific properties of $d_{\psi}$ and then show that these properties, by the 
definition of $d_{\psi}$, translates to the same properties of $J$.
The first of them is the 3-Points Property (3-pp) which is introduced in the following definition.
\begin{definition}[3-pp]
 Let $\mathcal{P}$ and $\mathcal{Q}$ be closed convex sets of finite measures. A function 
$d:\mathcal{P}\times \mathcal{Q}\to \mathbb{R}\cup \{-\infty,+\infty\}$ is said to satisfy the 3-points property (3-pp) if for a given $q\in\mathcal{Q}$
for which $d(p,q)<\infty$ $\forall p\in\mathcal{P}$, $\delta(p,p^{*})\le d(p,q)-d(p^{*},q)$ where 
$p^{*}=\underset{p\in\mathcal{P}}{\operatorname{argmin\text{ } }}d(p,q)$ and $\delta: \mathcal{P}\times\mathcal{P}\to \mathbb{R}_{+}$ with $\delta(p,p^{\prime})=0$ 
iff $p=p^{\prime}$.
\end{definition}

\begin{lemma}
$J$ satisfies 3-pp. 
\end{lemma}
\begin{proof}
 The proof is based on the works of \cite{wada03}. First, we will show that 3-pp is valid for $d_{\psi}(.,.)$ over $\mathcal{A}\times\mathcal{B}$.
As mentioned earlier, this is where the introduction of $d_{\psi}$ becomes useful and elegant.
Assume that $\mathbf{p} = \mathbb{A}(\mathbf{y}^{(l)})\in\mathcal{A}$ corresponding to
some $\mathbf{y}^{(l)}\in\mathcal{S}^{n}$, $\mathbf{q} = \mathbb{B}(\mathbf{y}^{(r)})\in\mathcal{B}$ 
corresponding to some $\mathbf{y}^{(r)}\in\mathcal{S}^{n}$
and $\mathbf{p}^{*} = \underset{\mathbf{p}\in\mathcal{A}}{\operatorname{argmin\text{ } }}
d_{\psi}(\mathbf{p},\mathbf{q}) = \underset{\mathbf{y}^{(l)}\in\mathcal{S}^{n}}{\operatorname{argmin\text{ } }}
J(\mathbf{y}^{(l)},\mathbf{y}^{(r)}) = \mathbb{A}\left({\mathbf{y}^{(l)}}^{*}\right)$ (the fact that the minimizers are 
just transformations of each other under $\mathbb{A}$ or $\mathbb{A}^{-1}$ follows directly from the separately strict convexity of 
$J$ and $d_{\psi}$). Therefore, 
\begin{eqnarray*}
 && d_{\psi}(\mathbf{p},\mathbf{q}) - d_{\psi}(\mathbf{p}^{*},\mathbf{q})\\ \nonumber
 &=& \psi(\mathbf{p}) - \psi(\mathbf{p}^{*}) -\langle \grad_{\psi}(\mathbf{q}), \mathbf{p} - \mathbf{p}^{*}\rangle \\ \nonumber 
 &=& \delta_{\psi}(\mathbf{p}, \mathbf{p}^{*}) + \langle \grad_{\psi}(\mathbf{p}^{*}) - \grad_{\psi}(\mathbf{q}), \mathbf{p} - \mathbf{p}^{*}\rangle
\end{eqnarray*}
where, $\delta_{\psi}:\mathcal{A}\times\mathcal{A}\rightarrow \mathbb{R}$ is defined as follows:
\begin{equation}
\delta_{\psi}(\mathbf{p}, \mathbf{p}^{*}) = \psi(\mathbf{p}) - \psi(\mathbf{p}^{*}) - \langle \grad_{\psi}(\mathbf{p}^{*}), \mathbf{p} - \mathbf{p}^{*}\rangle. 
\end{equation}
Since $\mathbf{p}^{*} = \underset{\mathbf{p}\in\mathcal{A}}{\operatorname{argmin\text{ } }}
d_{\psi}(\mathbf{p},\mathbf{q})$, $\langle \grad_{\mathbf{p}}d_{\psi}(\mathbf{p}^{*},\mathbf{q}),(\mathbf{p} - \mathbf{p}^{*})\rangle\ge 0$, 
then $\langle\grad_{\psi}(\mathbf{p}^{*})-\grad_{\psi}(\mathbf{q}),(\mathbf{p} - \mathbf{p}^{*})\rangle\ge 0$ which implies
$d_{\psi}(\mathbf{p},\mathbf{q}) - d_{\psi}(\mathbf{p}^{*},\mathbf{q})\ge \delta(\mathbf{p}, \mathbf{p}^{*})$.
Now, by some simple algebra, we can show $\delta_{\psi}(\mathbf{p}, \mathbf{p}^{*}) = \displaystyle\sum_{i=1}^{n}\left(\lambda+
\alpha\displaystyle\sum_{j=1;j\neq i}^{n}\right)d_{\phi}(\mathbf{y}_{i}^{(l)}, {\mathbf{y}_{i}^{(l)}}^{*})$.
By assumption, $d_{\phi}\big(\mathbf{y}_{i}^{(l)}, {\mathbf{y}_{i}^{(l)}}^{*}\big)\ge 0$ and hence 
$\delta_{\psi}(\mathbf{p}, \mathbf{p}^{*})\ge 0$ with $0$ achieved iff $\mathbf{y}^{(l)}={\mathbf{y}^{(l)}}^{*}$. 
If we define $\delta_{\psi}(\mathbf{p}, \mathbf{p}^{*})= \delta_{J}(\mathbb{A}^{-1}(\mathbf{p}), \mathbb{A}^{-1}(\mathbf{p}^{*}))$
then $\delta_{J}(\mathbf{y}^{(l)}, {\mathbf{y}^{(l)}}^{*})\ge 0$ with $0$ achieved iff $\mathbf{y}^{(l)}={\mathbf{y}^{(l)}}^{*}$.
Note that
\begin{equation}
\label{eqn:deltaJ}
 \delta_{J}(\mathbf{y}^{1,(l)}, {\mathbf{y}^{2, (l)}} = \displaystyle\sum_{i=1}^{n}\left(\lambda+
\alpha\displaystyle\sum_{j=1;j\neq i}^{n}\right)d_{\phi}(\mathbf{y}_{i}^{1, (l)}, {\mathbf{y}_{i}^{2, (l)}}).
\end{equation}
Therefore, following 3-pp of $d_{\psi}$ over $\mathcal{A}\times\mathcal{B}$, we can conclude that
\begin{equation}
J(\mathbf{y}^{(l)}, \mathbf{y}^{(r)}) - J({\mathbf{y}^{(l)}}^{*}, \mathbf{y}^{(r)}) \ge \delta_{J}(\mathbf{y}^{(l)}, {\mathbf{y}^{(l)}}^{*}), 
\end{equation}
which is the 3-pp for $J$.
\end{proof}

\begin{lemma}
 \label{delJlemma}
$\delta_{J}$ satisfies properties (c) and (f) mentioned in Section \ref{sec:convergence}.
\end{lemma}
\begin{proof}
\begin{enumerate}
 \item Since level sets of each of the terms in Eq. (\ref{eqn:deltaJ}) are bounded following the property (c) in Section \ref{sec:convergence},
we conclude that the level set $\{\mathbf{y}^{(r)}:\delta_{J}(\mathbf{y}^{(l)}, \mathbf{y}^{(r)})\le \ell\}$ for a given 
$\mathbf{y}^{(l)}\in\mathcal{S}^{n}$ is also bounded. 
\item We refer to Eq. (\ref{eqn:deltaJ}). As, $\mathbf{y}^{2,(l)} \to \mathbf{y}^{1,(l)}$, each of the 
$d_{\phi}(.,.)$'s goes to $0$ by the property (f) in Section \ref{sec:convergence}. Therefore, 
$\delta_{J}\to 0$ as $\mathbf{y}^{2,(l)} \to \mathbf{y}^{1,(l)}$. 
\end{enumerate}
\end{proof}

Next, 4-Points Property (4-pp) is introduced.
\begin{definition}[4-pp]
 Let $\mathcal{P}$ and $\mathcal{Q}$ be closed convex sets of finite measures. A function 
$d:\mathcal{P}\times \mathcal{Q}\to \mathbb{R}\cup \{-\infty,+\infty\}$ is said to satisfy 4-pp if for a given $p\in\mathcal{P}$, $d(p,q^{*}) \le \delta(p,p^{*}) + d(p,q)$ where 
$q^{*}=\underset{q\in\mathcal{Q}}{\operatorname{argmin\text{ } }}d(p^{*},q)$ and $\delta: \mathcal{P}\times\mathcal{P}\to \mathbb{R}_{+}$ with $\delta(p,p^{\prime})=0$ 
iff $p=p^{\prime}$.
\end{definition}

\begin{lemma}
$J$ satisfies 4-pp.  
\end{lemma}
\begin{proof}
Assume $\mathbf{u} = \mathbb{A}(\mathbf{y}^{1, (l)})\in\mathcal{A}$, 
$\mathbf{p} = \mathbb{A}(\mathbf{y}^{2, (l)})\in\mathcal{A}$, 
$\mathbf{q} = \mathbb{B}(\mathbf{y}^{3, (r)})\in\mathcal{B}$, and
$\mathbf{q}^{*} = \underset{\mathbf{q}\in\mathcal{B}}{\operatorname{argmin\text{ } }}
d_{\psi}(\mathbf{p},\mathbf{q}) = \mathbb{B}({\mathbf{y}^{4, (r)}}^{*})$. 
Here, $\mathbf{y}^{1, (l)}, \mathbf{y}^{2, (l)}, \mathbf{y}^{3, (r)}\in\mathcal{S}^{n}$ and
${\mathbf{y}^{4, (r)}}^{*} = \underset{\mathbf{y}^{(r)}\in\mathcal{S}^{n}}{\operatorname{argmin\text{ } }}
J(\mathbf{y}^{2, (l)},\mathbf{y}^{(r)})$. 
From the joint convexity of $d_{\psi}$ (established in Lemma \ref{lemmadpsi}) w.r.t both of its arguments we have:
\begin{equation}
 d_{\psi}(\mathbf{u}, \mathbf{v}) \ge d_{\psi}(\mathbf{p}, \mathbf{q}^{*}) 
+ \langle\grad_{\mathbf{p}}d_{\psi}(\mathbf{p}, \mathbf{q}^{*}), \mathbf{u} - \mathbf{p}\rangle
+ \langle\grad_{\mathbf{q}}d_{\psi}(\mathbf{p}, \mathbf{q}^{*}), \mathbf{v} - \mathbf{q}^{*}\rangle. 
\end{equation}
Since $\mathbf{q}^{*}$ minimizes $d_{\psi}(\mathbf{p},\mathbf{q})$ over $\mathbf{q}\in\mathcal{B}$, we have
$\langle\grad_{\mathbf{q}}d_{\psi}(\mathbf{p}, \mathbf{q}^{*}), \mathbf{v} - \mathbf{q}^{*}\rangle \ge 0$ which, in turn, implies:
\begin{equation*}
 d_{\psi}(\mathbf{u}, \mathbf{p}) - d_{\psi}(\mathbf{p}, \mathbf{q}^{*}) - \langle\grad_{\mathbf{p}}d_{\psi}(\mathbf{p}, \mathbf{q}^{*}), \mathbf{u} - \mathbf{p}\rangle\ge 0.
\end{equation*}
Now we have:
\begin{eqnarray*}
 && \delta_{\psi}(\mathbf{u}, \mathbf{p}) - d_{\psi}(\mathbf{u}, \mathbf{q}^{*}) \\ \nonumber
 &=& \psi(\mathbf{q}^{*}) - \psi(\mathbf{p}) -\langle \grad_{\psi}(\mathbf{q}^{*}), \mathbf{u} - \mathbf{q}^{*}\rangle 
- \langle \grad_{\psi}(\mathbf{p}), \mathbf{u} - \mathbf{p}\rangle \\ \nonumber
 &=& -d_{\psi}(\mathbf{p}, \mathbf{q}^{*}) -\langle \grad_{\psi}(\mathbf{p}) - \grad_{\psi}(\mathbf{q}^{*}), \mathbf{u} - \mathbf{p}\rangle \\ \nonumber
 &=& -d_{\psi}(\mathbf{p}, \mathbf{q}^{*}) -\langle \grad_{\mathbf{p}} d_{\psi}(\mathbf{p}, \mathbf{q}^{*}), \mathbf{u} - \mathbf{p}\rangle
\end{eqnarray*}
Combining the above two equations, we have,
\begin{equation}
\label{eqn:3pp}
 \delta_{\psi}(\mathbf{u}, \mathbf{p}) + d_{\psi}(\mathbf{u}, \mathbf{v}) \ge d_{\psi}(\mathbf{u}, \mathbf{q}^{*})
\end{equation}
Eq. (\ref{eqn:3pp}) gets translated for $J$ as follows (using definitions of $\delta_{\psi}$ and $d_{\psi}$): 
\begin{equation}
 \delta_{J}(\mathbf{y}^{1, (l)}, \mathbf{y}^{2, (l)}) + J(\mathbf{y}^{1, (l)}, \mathbf{y}^{3, (r)}) \ge J(\mathbf{y}^{1, (l)}, {\mathbf{y}^{4, (r)}}^{*})
\end{equation}
Hence, $J$ satisfies 4-pp.
\end{proof}

We now introduce the main theorem that establishes the convergence guarantee of \textbf{OAC\textsuperscript{3}}.

\begin{theorem}
 If $\mathbf{y}^{(l,t)} = \underset{\mathbf{y}^{(l)}\in \mathcal{S}^{n}}{\operatorname{argmin\text{ } }} J(\mathbf{y}^{(l)}, \mathbf{y}^{(r,t-1)})$, $\mathbf{y}^{(r,t)} = \underset{\mathbf{y}^{(r)}\in \mathcal{S}^{n}}{\operatorname{argmin\text{ } }} 
J(\mathbf{y}^{(l,t)}, \mathbf{y}^{(r)})$, then $\underset{t\to\infty}{\operatorname{lim}} J(\mathbf{y}^{(l,t)}, \mathbf{y}^{(r,t)}) = \underset{\mathbf{y}^{(l)}, \mathbf{y}^{(r)} \in \mathcal{S}^{n}}{\operatorname{inf}} J(\mathbf{y}^{(l)}, 
\mathbf{y}^{(r)})$. 
\end{theorem}
\begin{proof}
 The proof here follows the same line of argument as given in \cite{wada03} and \cite{egla98}. Since,
$\mathbf{y}^{(r,t+1)} = \underset{\mathbf{y}^{(r)}\in \mathcal{S}^{n}}{\operatorname{argmin\text{ } }}J(\mathbf{y}^{(l,t)}, \mathbf{y}^{(r)})$, we have, 
$J(\mathbf{y}^{(l,t)}, \mathbf{y}^{(r,t)}) - J(\mathbf{y}^{(l,t)}, \mathbf{y}^{(r,t+1)}) \ge 0$. By the 3-pp,  
$J(\mathbf{y}^{(l,t)}, \mathbf{y}^{(r,t+1)}) - J(\mathbf{y}^{(l,t+1)}, \mathbf{y}^{(r,t+1)}) \ge \delta_{J}(\mathbf{y}^{(l,t)}, \mathbf{y}^{(l,t+1)})$. 
Then,
\begin{eqnarray*}
&& J(\mathbf{y}^{(l,t)}, \mathbf{y}^{(r,t)}) - J(\mathbf{y}^{(l,t+1)}, \mathbf{y}^{(r,t+1)}) \\ \nonumber
&=& J(\mathbf{y}^{(l,t)}, \mathbf{y}^{(r,t)}) - J(\mathbf{y}^{(l,t)}, \mathbf{y}^{(r,t+1)}) + J(\mathbf{y}^{(l,t)}, 
\mathbf{y}^{(r,t+1)}) - J(\mathbf{y}^{(l,t+1)}, \mathbf{y}^{(r,t+1)})\\ \nonumber
&\ge& \delta_{J}(\mathbf{y}^{(l,t)}, \mathbf{y}^{(l,t+1)})\ge 0. 
\end{eqnarray*}
This implies that the sequence $J(\mathbf{y}^{(l,t)}, \mathbf{y}^{(r,t)})$ is non-increasing and non-negative. Let, 
$(\mathbf{y}^{(l,\infty)}, \mathbf{y}^{(r,\infty)}) = \underset{\mathbf{y}^{(l)}, 
\mathbf{y}^{(r)}\in \mathcal{S}^{n}}{\operatorname{argmin\text{ } }}J(\mathbf{y}^{(l)}, \mathbf{y}^{(r)})$. From 4-pp and 3-pp, we can 
derive the following two inequalities:
\begin{eqnarray}
 J(\mathbf{y}^{(l,\infty)}, \mathbf{y}^{(r,t+1)}) &\le& \delta_{J}(\mathbf{y}^{(l,\infty)}, \mathbf{y}^{(l,t)}) + J(\mathbf{y}^{(l,\infty)}, \mathbf{y}^{(r,\infty)}) \\
 \delta_{J}(\mathbf{y}^{(l,\infty)}, \mathbf{y}^{(l,t+1)}) &\le& J(\mathbf{y}^{(l,\infty)}, \mathbf{y}^{(r,t+1)}) - J(\mathbf{y}^{(l,t+1)}, \mathbf{y}^{(r,t+1)}).  
\end{eqnarray}
Combining the above two inequalities, we get:
\begin{equation}
\label{J5pp}
 \delta_{J}(\mathbf{y}^{(l,\infty)}, \mathbf{y}^{(l,t)}) - \delta_{J}(\mathbf{y}^{(l,\infty)}, \mathbf{y}^{(l,t+1)}) \ge J(\mathbf{y}^{(l,t+1)}, \mathbf{y}^{(r,t+1)}) - J(\mathbf{y}^{(l,\infty)}, \mathbf{y}^{(r,\infty)})\ge 0, 
\end{equation}
which is the 5-points property (5-pp) of $J$. From (\ref{J5pp}), the sequence $\delta_{J}(\mathbf{y}^{(l,\infty)}, \mathbf{y}^{(l,t)})$ is non-increasing and non-negative. 
Therefore, it must have a limit (from the Monotone Convergence Theorem) and consequently the left hand side of (\ref{J5pp})  
approaches 0 as $t\to \infty$. Hence, 
$\underset{t\to\infty}{\operatorname{lim}} J(\mathbf{y}^{(l,t)}, \mathbf{y}^{(r,t)}) =  J(\mathbf{y}^{(l,\infty)}, \mathbf{y}^{(r,\infty)})$ (by the 
Pinching Theorem). 

Finally, we must show that $\mathbf{y}^{(l,t)}$ and $\mathbf{y}^{(r,t)}$ themselves converge. From the boundedness 
of $\delta_{J}(\mathbf{y}^{(l,\infty)}, \mathbf{y}^{(l,t)})$ (established in Lemma \ref{delJlemma}), it follows 
that $\mathbf{y}^{(l,t)}$ is bounded. Therefore, it has a convergent 
subsequence $\{\mathbf{y}^{(l,t_{i})}\}$ -- the limit of which can be denoted by $\mathbf{y}^{0,(l)}$ (by
the Bolzano-Weierstrass Theorem). Similarly, it can be 
shown that the subsequence $\{\mathbf{y}^{(r,t_{i})}\}$ also converges to some limit. Let that limit be denoted by
$\mathbf{y}^{0,(r)}$. By the lower-semi-continuity of $J$ (established in Lemma \ref{Jlemma}), we have:
\begin{equation}
 J(\mathbf{y}^{0,(l)}, \mathbf{y}^{0,(r)}) \le \underset{i}{\operatorname{lim \text{ }inf }} J(\mathbf{y}^{(l, t_{i})}, \mathbf{y}^{(r, t_{i})}) 
= J(\mathbf{y}^{(l, \infty)}, \mathbf{y}^{(r, \infty)}).
\end{equation}
We denote $\mathcal{Y}_{l}^{\infty} = \{\mathbf{y}^{(l)}:\underset{\mathbf{y}^{(l)}, 
\mathbf{y}^{(r)}\in \mathcal{S}^{n}}{\operatorname{arg\text{ }min\text{ } }}J(\mathbf{y}^{(l)}, \mathbf{y}^{(r)}) \}$
and $\mathcal{Y}_{r}^{\infty} = \{\mathbf{y}^{(r)}:\underset{\mathbf{y}^{(l)}, 
\mathbf{y}^{(r)}\in \mathcal{S}^{n}}{\operatorname{arg\text{ }min\text{ } }}J(\mathbf{y}^{(l)}, \mathbf{y}^{(r)}) \}$.
Therefore, from the joint strict convexity of $J$, we should have  
$\mathcal{Y}_{l}^{\infty} = \{\mathbf{y}^{0,(l)}\} = \{\mathbf{y}^{(l, \infty)}\}$ and $\mathcal{Y}_{r}^{\infty} = \{\mathbf{y}^{0,(r)}\} 
=\{\mathbf{y}^{(r, \infty)}\}$.

To prove the convergence of the entire sequence, we apply the same logic as above 
with $\mathbf{y}^{(l, \infty)}$ replaced by $\mathbf{y}^{0, (l)}$. Then the sequence
$\{\delta_{J}(\mathbf{y}^{0, (l)}, \mathbf{y}^{(l, t)})\}$ is bounded and 
non-increasing and by using Lemma \ref{delJlemma}, we conclude that it has a convergent subsequence $\{\delta_{J}(\mathbf{y}^{0, (l)}, \mathbf{y}^{(l, t_{i})})\}$ that goes to $0$ as
$\mathbf{y}^{(l, t_{i})}\to \mathbf{y}^{0, (l)}$. This, from Monotone Convergence Theorem, implies that 
$\{\delta_{J}(\mathbf{y}^{0, (l)}, \mathbf{y}^{(l, t)})\}\to 0$ and again using Lemma \ref{delJlemma}, 
we can conclude that $\mathbf{y}^{(l, t)}\to \mathbf{y}^{0, (l)}$.
Since $\mathbf{y}^{(r, t)}$ is also bounded, it should have a convergent subsequence (by
the Bolzano-Weierstrass Theorem). We denote this limit by $\mathbf{y}^{(0), (r)}$. Again, by the lower-semi-continuity
of $J$, we have:
\begin{equation}
 J(\mathbf{y}^{0, (l)},\mathbf{y}^{(0), (r)}) \le J(\mathbf{y}^{(l,\infty)},\mathbf{y}^{(r, \infty)}).
\end{equation}
Hence, $\mathbf{y}^{(0), (r)} = \underset{\mathbf{y}^{(r)}\in \mathcal{S}^{n}}
{\operatorname{arg\text{ }min\text{ } }}J(\mathbf{y}^{(l)}, \mathbf{y}^{(r, \infty)})$ and $\mathbf{y}^{(r, t)}\to\mathbf{y}^{(0),(r)}
=\mathbf{y}^{0,(r)}$. 
\end{proof}
 
There is another interesting aspect of $J$ that was discovered in \cite{subi11} for a slightly 
different objective function with KL divergence used as a loss function. 
The same property also holds for $J$ if the loss function is constructed from the assumed family of Bregman divergences.  
This property is concerned with the equality of solutions of $J$ and $J_{0}$ and explores under what 
conditions these two objectives become equal.
To establish the theorem that explores this condition, the following lemmata are essential.

\begin{lemma}
 If $\mathbf{y}^{(r)} = \mathbf{y}^{(l)} = \mathbf{y}$ then $J_{0} = J$. 
\end{lemma}
\begin{proof}
 This proof immediately follows from the definitions of $J_{0}$ and $J$ in Eq. (\ref{eqn:5}) and Eq. (\ref{eqn:mainobj}) respectively.
\end{proof}

\begin{lemma}
\label{subilemma1}
$\underset{(\mathbf{y}^{(l)}, \mathbf{y}^{(r)})\in \mathcal{S}^{n}\times \mathcal{S}^{n}}{\operatorname{arg\text{ }min\text{ }}} 
J(\mathbf{y}^{(l)}, \mathbf{y}^{(r)}; \lambda = 0) 
\le \underset{\mathbf{y}\in \mathcal{S}^{n}} {\operatorname{arg\text{ }min\text{ }}} J_{0}(\mathbf{y})$.
\end{lemma}

\begin{proof}
\begin{equation*}
\underset{\mathbf{y}\in \mathcal{S}^{n}}{\operatorname{min\text{ }}} J_{0}(\mathbf{y}) = 
\underset{\left(\mathbf{y}^{(l)}, \mathbf{y}^{(r)}\right)\in \mathcal{S}^{n}\times \mathcal{S}^{n}; \mathbf{y}^{(r)} = \mathbf{y}^{(l)}}
{\operatorname{min\text{ }}} J(\mathbf{y}^{(l)}, \mathbf{y}^{(r)}; \lambda = 0) 
\ge \underset{\left(\mathbf{y}^{(l)}, \mathbf{y}^{(r)}\right)\in \mathcal{S}^{n}\times \mathcal{S}^{n}}
 {\operatorname{min\text{ } }} J(\mathbf{y}^{(l)}, \mathbf{y}^{(r)}; \lambda = 0)
\end{equation*}
The last step is due to the fact that the unconstrained minima is never larger than the constrained minima.
\end{proof}

\begin{lemma}
\label{subilemma2}
Given any $\mathbf{y}^{(l)}$, $\mathbf{y}^{(r)}$, $\mathbf{y}\in\mathcal{S}^{n}$  such that $\mathbf{y}^{(l)}$, $\mathbf{y}^{(r)}$, $\mathbf{y} > \mathbf{0}$
and $\mathbf{y}^{(l)}\neq \mathbf{y}^{(r)}$ (\emph{i.e.} not all components are equal) then there exists a finite $\lambda$
such that $J(\mathbf{y}^{(l)}, \mathbf{y}^{(r)}) \ge J(\mathbf{y}, \mathbf{y}) = J_{0}(\mathbf{y})$. 
\end{lemma}

\begin{proof}
For $J(\mathbf{y}^{(l)}, \mathbf{y}^{(r)}) \ge J(\mathbf{y}, \mathbf{y})$, we should have:
\begin{eqnarray*}
&& \left[\displaystyle\sum_{i=1}^{n}d_{\phi}(\boldsymbol{\pi}_{i}, \mathbf{y}_{i}^{(r)})
+ \alpha\displaystyle\sum_{i,j=1}^{n} s_{ij}d_{\phi}(\mathbf{y}_{i}^{(l)}, \mathbf{y}_{j}^{(r)})
+ \lambda\displaystyle\sum_{i=1}^{n}d_{\phi}(\mathbf{y}_{i}^{(l)}, \mathbf{y}_{i}^{(r)})\right] - 
J(\mathbf{y}, \mathbf{y})\ge 0 \\ 
&\Rightarrow& \lambda \ge \frac{J(\mathbf{y}, \mathbf{y}) - \displaystyle\sum_{i=1}^{n}d_{\phi}(\boldsymbol{\pi}_{i}, \mathbf{y}_{i}^{(r)})
- \alpha\displaystyle\sum_{i,j=1}^{n} s_{ij}d_{\phi}(\mathbf{y}_{i}^{(l)}, \mathbf{y}_{j}^{(r)})}
{\displaystyle\sum_{i=1}^{n}d_{\phi}(\mathbf{y}_{i}^{(l)}, \mathbf{y}_{i}^{(r)})}\\ 
&\Rightarrow& \lambda \ge \frac{J_{0}(\mathbf{y}) - J(\mathbf{y}^{(l)}, \mathbf{y}^{(r)};\lambda=0)}
{\displaystyle\sum_{i=1}^{n}d_{\phi}(\mathbf{y}_{i}^{(l)}, \mathbf{y}_{i}^{(r)})} \ge 0.  
\end{eqnarray*}
where the last inequality follows from Lemma \ref{subilemma1}.
\end{proof}

The theorem that formulates the conditions for equality of solutions of $J$ and $J_{0}$ is given below:
\begin{theorem}[Equality of Solutions of $J$ and $J_{0}$]
\label{app_subithm2}
Let $\mathbf{y}^{*} = \underset{\mathbf{y}\in \mathcal{S}^{n}}
{\operatorname{arg\text{ }min\text{ } }} J_{0}(\mathbf{y})$ and 
$({\mathbf{y}^{\tilde{\lambda}, (l)}}^{*}, {\mathbf{y}^{\tilde{\lambda}, (r)}}^{*}) = \underset{\mathbf{y}\in \mathcal{S}^{n}}{\operatorname{arg\text{ }min\text{ } }}
J (\mathbf{y}^{(l)}, \mathbf{y}^{(r)}; \tilde{\lambda})$ 
for an arbitrary $\lambda = \tilde{\lambda} > 0$. Then there exists a finite $\hat{\lambda}$ such that at 
convergence of OAC\textsuperscript{3}, we have $\mathbf{y}^{*} = {\mathbf{y}^{\hat{\lambda}, (l)}}^{*} = {\mathbf{y}^{\hat{\lambda}, (r)}}^{*}$. 
Further, if ${\mathbf{y}^{\tilde{\lambda}, (l)}}^{*} \neq {\mathbf{y}^{\tilde{\lambda}, (r)}}^{*}$, then 
\begin{equation*}
\hat{\lambda} \ge \frac{J_{0}(\mathbf{y}^{*}) - J({\mathbf{y}^{\tilde{\lambda},(l)}}^{*}, {\mathbf{y}^{\tilde{\lambda},(r)}}^{*};\lambda=0)}
{\displaystyle\sum_{i=1}^{n}d_{\phi}(\mathbf{y}_{i}^{\tilde{\lambda}, (l)}, \mathbf{y}_{i}^{\tilde{\lambda}, (r)})}
\end{equation*}
and if ${\mathbf{y}^{\tilde{\lambda},(l)}}^{*} = {\mathbf{y}^{\tilde{\lambda},(r)}}^{*}$, then $\hat{\lambda} \ge \tilde{\lambda}$.
\end{theorem}

\begin{proof}
 If ${\mathbf{y}^{\tilde{\lambda},(l)}}^{*} = {\mathbf{y}^{\tilde{\lambda},(r)}}^{*}$, then from the strict convexity of 
both $J_{0}$ and $J$,  
$J_{0}(\mathbf{y}^{*}) = J({\mathbf{y}^{\tilde{\lambda},(l)}}^{*}, {\mathbf{y}^{\tilde{\lambda},(r)}}^{*};\lambda=0)$.
Also, since for any $\mathbf{y}^{(l)} \neq \mathbf{y}^{(r)}$, 
$J(\mathbf{y}^{(l)}, \mathbf{y}^{(r)};\hat{\lambda}) > J(\mathbf{y}^{(l)}, \mathbf{y}^{(r)};\tilde{\lambda})$, whenever
$\hat{\lambda} \ge \tilde{\lambda}$, then $\forall \hat{\lambda} \ge \tilde{\lambda}$
$J_{0}(\mathbf{y}^{*}) = J({\mathbf{y}^{\hat{\lambda},(l)}}^{*}, {\mathbf{y}^{\hat{\lambda},(r)}}^{*};\lambda=0)$.
Also, if ${\mathbf{y}^{\tilde{\lambda},(l)}}^{*} \neq {\mathbf{y}^{\tilde{\lambda},(r)}}^{*}$, 
then from Lemma \ref{subilemma2}, if
\begin{equation*}
 \infty > \hat{\lambda} \ge \frac{J_{0}(\mathbf{y}^{*}) - J({\mathbf{y}^{\tilde{\lambda},(l)}}^{*}, {\mathbf{y}^{\tilde{\lambda},(r)}}^{*};\lambda=0)}
{\displaystyle\sum_{i=1}^{n}d_{\phi}(\mathbf{y}_{i}^{\tilde{\lambda}, (l)}, \mathbf{y}_{i}^{\tilde{\lambda}, (r)})}
\end{equation*}
then it is guaranteed that ${\mathbf{y}^{\hat{\lambda},(l)}}^{*} = {\mathbf{y}^{\hat{\lambda},(r)}}^{*}$.
\end{proof}
\section{Proof for Analysis of Rate of Convergence}
\label{app:rateconv}
\begin{lemma}
 $\mathcal{H} = \grad^{2}J$ is positive definite over the domain of $J$ under the assumption 
$\displaystyle\sum_{i=1}^{n}\displaystyle\sum_{\ell=1}^{k}\pi_{i\ell}> 0$ when KL or generalized I divergence is used as a 
Bregman divergence.
\end{lemma}
\begin{proof}
 Assume $\mathbf{z} = \left(\left({\mathbf{y}_{i}^{(l)}}^{\dagger}\right)_{i=1}^{n},  \left({\mathbf{y}_{i}^{(r)}}^{\dagger}\right)_{i=1}^{n} \right)^{\dagger}$.
Now, 
\begin{eqnarray}
&& \mathbf{z}^{\dagger}\mathcal{H}\mathbf{z} \\ \nonumber
&=& \displaystyle\sum_{i=1}^{n}{\mathbf{y}_{i}^{(l)}}^{\dagger} \grad_{\mathbf{y}_{i}^{(l)}, \mathbf{y}_{i}^{(l)}}\mathbf{y}_{i}^{(l)}
+ \displaystyle\sum_{j=1}^{n}{\mathbf{y}_{j}^{(r)}}^{\dagger} \grad_{\mathbf{y}_{j}^{(r)}, \mathbf{y}_{j}^{(r)}}\mathbf{y}_{j}^{(r)}
+ 2 \displaystyle\sum_{i,j=1; i\neq j}^{n}{\mathbf{y}_{i}^{(l)}}^{\dagger} \grad_{\mathbf{y}_{i}^{(l)}, \mathbf{y}_{i}^{(r)}}\mathbf{y}_{j}^{(r)}\\ \nonumber
&+& 2 \displaystyle\sum_{i=1}^{n}{\mathbf{y}_{i}^{(l)}}^{\dagger} \grad_{\mathbf{y}_{i}^{(l)}, \mathbf{y}_{i}^{(r)}}\mathbf{y}_{i}^{(r)} \\ \nonumber
&=& \displaystyle\sum_{i=1}^{n}\big(\alpha\displaystyle\sum_{j=1;j\neq i}^{n}s_{ij}+\lambda\big) \displaystyle\sum_{\ell=1}^{k}\mathbf{y}_{i\ell}^{(l)}
+ \displaystyle\sum_{j=1}^{n}\displaystyle\sum_{\ell=1}^{k}\big(\pi_{j\ell} + \alpha \displaystyle\sum_{i=1;i\neq j}^{n}s_{ij} y_{i\ell}^{(l)} 
+\lambda y_{j\ell}^{(l)}\big) - 2\lambda\displaystyle\sum_{i=1}^{n}\displaystyle\sum_{\ell=1}^{k}y_{i\ell}^{(l)}\\ \nonumber 
&-& 2\alpha\displaystyle\sum_{i, j=1; i \neq j}^{n}s_{ij}\displaystyle\sum_{\ell=1}^{k}y_{i\ell}^{(l)}\\ \nonumber
&=& \displaystyle\sum_{i=1}^{n}\displaystyle\sum_{\ell=1}^{k}\pi_{i\ell}> 0. 
\end{eqnarray}
Therefore, if $\displaystyle\sum_{i=1}^{n}\displaystyle\sum_{\ell=1}^{k}\pi_{i\ell}> 0$, $\grad^{2}J$ is positive definite over the domain of $J$.
\end{proof}

\begin{acks}
We are grateful to Luiz F. S. Coletta for running the experiments described in Section \ref{Sensitivity_Analysis}. We also thank Ambuj Tewari and Ali Jalali for pointing us 
to relevant literarure for analyzing the rate of convergence of the optimization framework.
\end{acks}

\bibliographystyle{acmsmall}
\bibliography{bibfile_TKDD}

\received{April 2012}{April 2012}{April 2012}


\end{document}